\pdfoutput=1

\documentclass{article}
\usepackage{colm2024_conference}

\usepackage{url}
\usepackage{booktabs}
\usepackage{hyperref}
\usepackage{graphicx}
\usepackage{microtype}
\usepackage{multicol}
\usepackage{xcolor}   
\usepackage{multirow}
\usepackage{lipsum}
\usepackage{enumitem}
\usepackage{pifont}
\usepackage{colortbl}
\usepackage{tabularx}
\usepackage{caption} 
\usepackage{float}
\usepackage{amsmath,amsfonts,amssymb,bbm}
\usepackage{amsthm}
\usepackage{algorithm}
\usepackage{algpseudocode}
\usepackage{subfig}
\usepackage{graphicx}

\newtheorem{proposition}{Proposition}

\makeatletter
\def\@fnsymbol#1{\ensuremath{\ifcase#1\or \dagger \or  \ddagger\or
   \mathsection\or  \text{*}\or \mathparagraph \or  \| \or **\or \dagger\dagger
   \or \ddagger\ddagger \else\@ctrerr\fi}}
\makeatother
\renewcommand{\thefootnote}{\fnsymbol{footnote}}

\title{Breaking Entropy Bounds: Accelerating RL Training via MTP with Rejection Sampling}

\author{
\textbf{%
Yucheng Li\footnotemark[1] \quad Huiqiang Jiang$^{\dagger\ddagger}$ \quad Yang Xu \quad Jianxin Yang \quad Yi Zhang \\ Yizhong Cao \quad Yuhao Shen \quad Fan Zhou \quad Rui Men \quad Jianwei Zhang \quad An Yang \quad Bowen Yu \quad Bo Zheng \quad Fei Huang \quad Junyang Lin \quad Dayiheng Liu \quad Jingren Zhou} \\
\vspace{1.5mm}
Qwen Team, Alibaba Inc.
}

\begin{document}
\maketitle

\footnotetext[1]{Equal contribution. $^{\ddagger}$ Corresponding author.}
\renewcommand{\thefootnote}{\arabic{footnote}}
\setcounter{footnote}{0}

\begin{abstract}

Reinforcement learning (RL) has become a key component in modern large language models, yet the rollout stage remains the key bottleneck in RL training pipelines.
Although Multi-Token Prediction (MTP) offers a natural solution to accelerate rollouts through speculative decoding, many studies have observed that MTP acceptance rates degrade significantly during RL training, leading to limited speedup performance.
To address this bottleneck, we present \textbf{\textit{Bebop}}, a systematic study of MTP in LLM post-training, and offer practical recipes to integrate MTP into large-scale RL pipelines.
First, we reveal that the MTP acceptance rate is fundamentally bounded by the fluctuation of model entropy, which demonstrates a clear negative linear relationship with the rise of entropy in the RL stage (\S\ref{sec:formulation}).
Second, we show that \textbf{probabilistic rejection sampling} largely alleviates the disturbance introduced by entropy in RL compared to greedy draft sampling. We further identify that the conventional MTP training objectives (cross-entropy or KL) are suboptimal in such settings, and therefore we propose a novel \textbf{end-to-end TV loss} that directly optimizes multi-step rejection sampling acceptance rate, yielding ${\sim}10\%$ acceptance rate improvements, achieving up to 95\% acceptance rates and up to 25\% extra inference throughput gains across mathematical reasoning, code generation, and agentic tasks (\S\ref{sec:mtp_training}).
Third, we test various online MTP training strategies during RL and show that pre-RL MTP training with e2e TV loss and rejection sampling achieves a consistent acceptance rate and speedup throughout the entire RL, eliminating the need for costly online MTP updating (\S\ref{sec:mtp_adapt}).
We provide extensive experiments and analysis that validate our findings. Experimental results show our method achieves up to $1.8\times$ end-to-end acceleration in async RL training of Qwen3.5, Qwen3.6, and Qwen3.7 models.

\end{abstract}

\begin{figure*}[htb]
    \vspace{-14pt}
    \centering
    \subfloat[Entropy vs.\ Accept Length]{
      \label{sfig:entropy_vs_accept_length}
      \includegraphics[width=0.5\columnwidth]{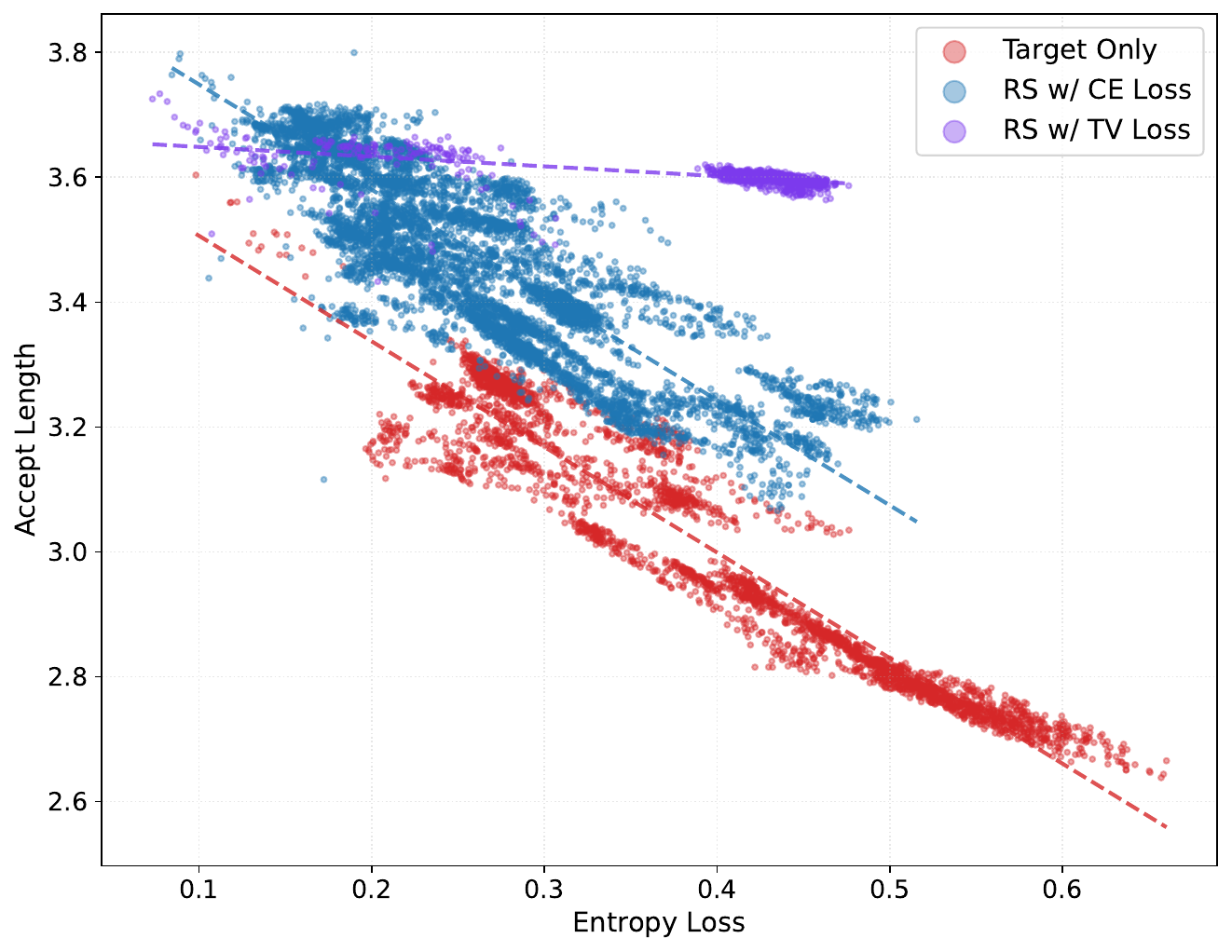}}
    \subfloat[Draft/Target Distribution]{
      \label{sfig:draft_target_distribution}
      \includegraphics[width=0.3\columnwidth]{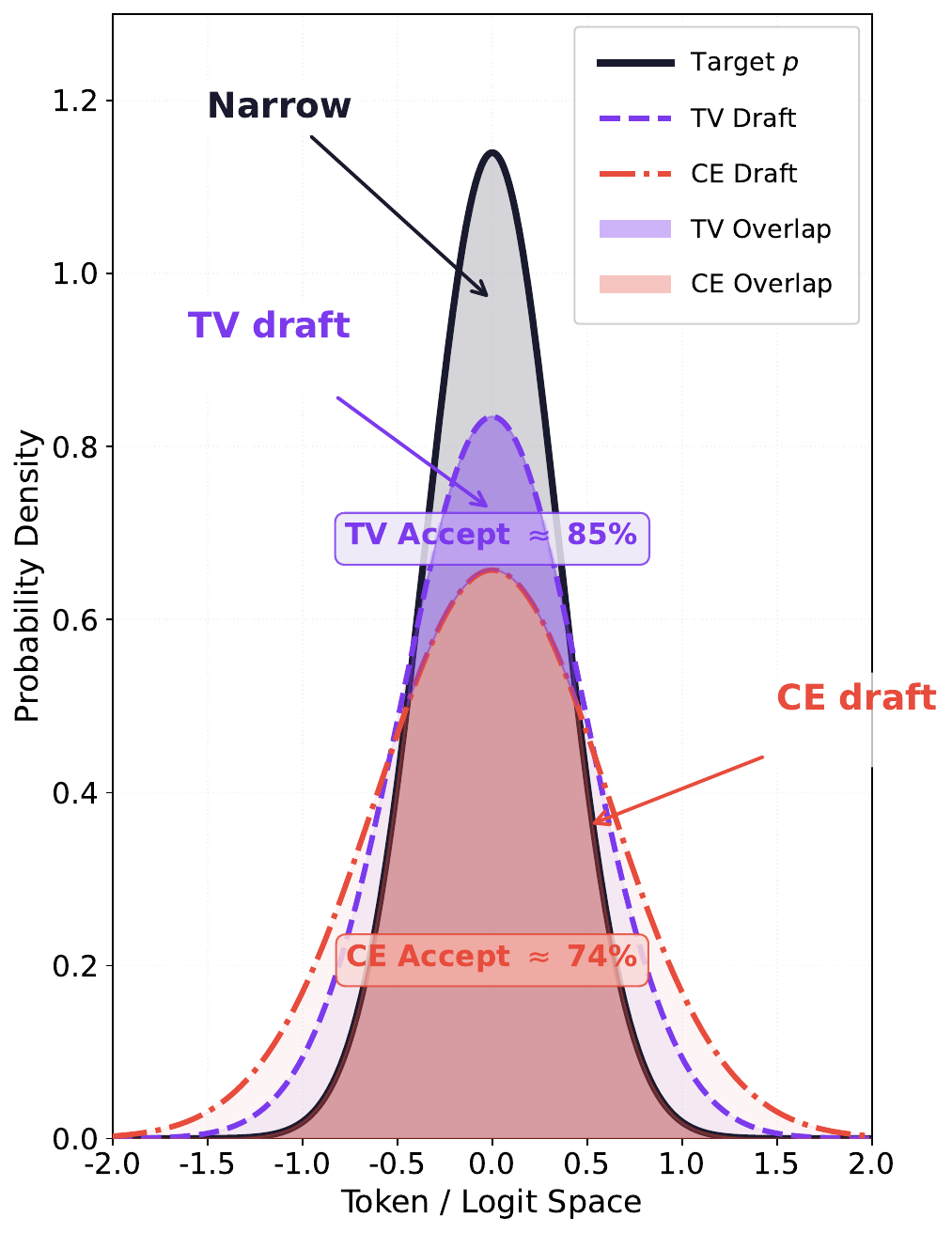}}
    \caption{
        (a) MTP acceptance rates degrade linearly with policy entropy fluctuation in RL; training MTP with our novel e2e TV loss largely eliminates this entropy dependence under rejection sampling.
        Each point represents the mean entropy and accept length at one RL step across different-size Qwen3.5, 3.6 and 3.7 training runs in various tasks.
        (b) The TV-trained MTP achieves substantially better distributional overlap with the policy model, yielding superior acceptance rate and speedup.
    }
    \label{fig:performance_gain}
\end{figure*}
  
\section{Introduction}
\label{sec:intro}

Reinforcement learning (RL) has become a key paradigm in modern large language model (LLM) training~\citep{gpt55,claude5,qwen37,deepseek-v4,glm5.1,k2.6,minimax2.5}. However, RL training for LLMs is computationally expensive, with the end-to-end time heavily dominated by inference rollouts in both single- and multi-turn settings. Although recent progress in asynchronous RL frameworks~\citep{fu2025areal,wang2025roll,slime2025} can partially alleviate long-tail latency issues, rollout costs remain the primary bottleneck in RL training.
Multi-Token Prediction (MTP) has recently gained prominence as a scalable speculative decoding paradigm to accelerate LLM inference \citep{dpsk-v3,qwen3.5}.
This naturally raises the question: can MTP be effectively leveraged to accelerate RL training for LLMs?

We conduct extensive experiments and show that using MTP directly in RL training often suffers from a significant decline in acceptance rates and therefore leads to limited speedup. Specifically, there are two factors that may affect MTP acceptance rates during RL: 1) to encourage exploration, the policy model often maintains a rather large entropy--or even shows a gradually increasing entropy curve, which makes it harder to predict draft tokens, degrading the acceptance rate; 2) the weight updates of the policy model cause distribution mismatch between the policy model and the MTP module (frozen in RL training), that may affect the acceptance rate. Through our theoretical analysis and empirical decomposition (\S\ref{sec:formulation}), we show that \textbf{entropy is the dominant factor} driving acceptance rate degradation, while the mismatch introduced by policy updates remains negligible (Fig.~\ref{fig:decomposition_delta}). To tackle the entropy bound challenge and ensure the speedup of MTP, recent works~\citep{chen2025respec,li2025mtp,minimax2026forge} have proposed online MTP training during RL to mitigate this degradation, yet this approach introduces significant memory and latency overhead and yields limited improvements in many RL tasks.

In this paper, we introduce \textbf{\textit{Bebop}}\footnote{\textbf{B}reaking \textbf{E}ntropy \textbf{B}ounds for \textbf{O}ptimal \textbf{P}rediction} and show that using \textbf{probabilistic rejection sampling}\footnote{We have released our implementation at \url{https://github.com/sgl-project/sglang/pull/26312}.} instead of the common greedy target-only sampling\footnote{Target-only sampling means the verification of speculative decoding uses only target probability without caching draft probability: it selects draft tokens via $\arg\max$ and uses $1$ as $q(\hat{y})$ in the verification.} largely mitigates the acceptance rate degradation driven by policy entropy fluctuation (\S\ref{sec:mtp-with-rs}) and provides a large improvement in acceptance rate. The key insight is that target-only acceptance is fundamentally capped by $\max_y p(y)$, which decreases directly as entropy rises, whereas rejection sampling acceptance equals the full distributional overlap $\sum_v \min(p(v), q(v))$ and is therefore much less sensitive to entropy shifts. We further identify that existing MTP training objectives, such as cross-entropy (CE) or KL divergence, are suboptimal for rejection sampling: CE/KL only indirectly improve the distributional overlap that determines rejection sampling acceptance. This motivates us to propose a novel \textbf{end-to-end TV loss} that optimizes the joint multi-step overlap that directly improves rejection sampling acceptance rate.

\textit{Bebop} produces MTP models that maintain consistent acceptance rates throughout the entire RL training process. These rates remain largely invariant to entropy changes. \textit{Bebop} achieves this stability using only a lightweight pre-RL MTP training phase with an e2e TV loss, paired with rejection sampling during rollouts, eliminating the need for MTP co-training during RL.

Specifically, we make the following contributions:

\begin{itemize}[leftmargin=10mm, itemsep=2mm]

\item \textbf{Entropy Constraints on MTP Acceptance} (\S\ref{sec:formulation}).
We show that MTP acceptance rates are fundamentally constrained by the target model's entropy in RL training, exhibiting a clear negative linear relationship across diverse tasks and models.
We further show that rejection sampling largely improves the acceptance rate in RL, as its acceptance depends on policy-draft overlap and is less sensitive to entropy shifts.

\item \textbf{End-to-End TV Loss for MTP Training} (\S\ref{sec:mtp_training}).
We identify that CE/KL-trained MTP produces suboptimal results in rejection sampling, and thereby introduce a novel end-to-end TV loss that directly optimizes the multi-step rejection sampling acceptance rate.
We show that the e2e TV loss ensures stable training, produces inherently entropy-invariant MTP, and yields an extra ${\sim}10\%$ improvement in acceptance rate.

\item \textbf{MTP Adaptation Strategy for RL} (\S\ref{sec:mtp_adapt}).
We show that with a lightweight pre-RL MTP training with e2e TV loss and rejection sampling, our MTP module provides consistent acceptance rates throughout the entire RL training. The other factor, policy-draft mismatch driven by policy updates, is negligible, which eliminates the need for costly MTP online training during RL.

\item \textbf{Extensive Empirical Validation and Analysis} (\S\ref{sec:exp}, \S\ref{sec:discussion}).
Through large-scale experiments with Qwen3.5, 3.6, and 3.7 models on reasoning, coding, and various agentic tasks, we validate \textit{Bebop} and provide practical recipes for integrating MTP into RL pipelines, achieving up to $1.8\times$ end-to-end acceleration of async RL pipelines. We further analyze how TV loss shapes draft distributions, the robustness of rejection sampling under policy updates, and the effects of temperature and generation length on acceptance rates.

\end{itemize}
\section{Preliminaries}
\label{sec:background}

\subsection{Multi-Token Prediction and Speculative Decoding}

As an effective paradigm of speculative decoding~\citep{leviathan2023fast,chen2023accelerating}, Multi-Token Prediction (MTP) augments autoregressive LLMs with lightweight \textit{draft heads} that sequentially predict multiple future tokens~\citep{gloeckle2024bmtp,dpsk-v3,qwen3}.
Let $p(\cdot | x, y_{<t})$ denote the target (backbone) model's next-token distribution at position $t$, and $q(\cdot | x, y_{<t})$ denote the draft head's predicted distribution.
During inference, MTP operates in a \textit{draft-then-verify} paradigm: a chain of $\gamma$ draft heads sequentially proposes candidate tokens $\hat{y}_{t+1}, \ldots, \hat{y}_{t+\gamma}$, where each head takes the previous head's hidden state as input; the $\gamma$ candidates are then verified against the target model in a single forward pass.

The expected number of accepted tokens per verification step, which we call the \textit{acceptance length}, directly determines the inference throughput.
This acceptance length depends on the specific \textit{acceptance methods} used during verification, detailed in the following section.

\subsection{Acceptance Methods}

In speculative decoding, two acceptance methods are commonly used: \textit{Target-Only Sampling} and \textit{Rejection Sampling}. Fig.~\ref{sfig:rs_boundary} illustrates the acceptance rate distributions of representative models under each method.

\paragraph{Target-Only Sampling.}
Under target-only sampling, the draft token is selected greedily as $\hat{y} = \arg\max_y q(y)$ and accepted with probability $p(\hat{y})$, using only the target model's probability.
The single-step acceptance rate is:
\begin{align}
\label{equ:target_only}
\alpha^{\mathrm{TO}} = p(\hat{y}) = p\!\left(\arg\max_y\, q(y)\right).
\end{align}
If rejected, the output token is resampled from the residual distribution $p_{\mathrm{resid}}(y) \propto p(y)\,\mathbbm{1}[y \neq \hat{y}]$, ensuring the overall output distribution remains unbiased.
Notably, for draft models with relatively low acceptance rates, target-only sampling can yield higher throughput than rejection sampling, as the simpler acceptance criterion avoids the overhead of caching and computing the draft probability vectors.

\paragraph{Rejection Sampling.}
Under rejection sampling \citep{leviathan2023fast,chen2023accelerating}, a draft token $\hat{y} \sim q(\cdot)$ is accepted with probability $\min\!\left(1, \, p(\hat{y}) / q(\hat{y})\right)$.
The expected single-step acceptance rate is:
\begin{align}
\label{equ:reject_sample}
\alpha^{\mathrm{RS}} = \mathbb{E}_{\hat{y} \sim q}\!\left[\min\!\left(1,\;\frac{p(\hat{y})}{q(\hat{y})}\right)\right] = \sum_{y} \min\bigl(p(y),\;q(y)\bigr) = 1 - d_{\mathrm{TV}}(p, q),
\end{align}
where $d_{\mathrm{TV}}(p, q) = \frac{1}{2}\sum_{y} |p(y) - q(y)|$ is the Total Variation distance~\citep{levin2017markov}.
This method provides an \textit{unbiased} guarantee: the output distribution is exactly the target distribution $p$, regardless of the draft quality.

\subsection{Reinforcement Learning for LLMs}

We consider the standard RL framework for LLMs, where a policy $\pi_\theta$ (the LLM) generates trajectories $y$ to prompts $x \sim \mathcal{D}$ and receives scalar rewards $R(x, y)$.
We adopt GRPO~\citep{grpo}, which samples a group of $G$ trajectories $\{y_1, \ldots, y_G\}$ from the rollout policy $\pi_{\theta_{\mathrm{old}}}$ for each prompt, and optimizes the clipped surrogate objective:
\begin{align}
\mathcal{J}(\theta) = \mathbb{E}_{x \sim \mathcal{D}} \left[ \frac{1}{G} \sum_{i=1}^{G} \frac{1}{|y_i|} \sum_{t=1}^{|y_i|} \min\!\left( r_{i,t}\, \hat{A}_i,\; \mathrm{clip}(r_{i,t}, 1{-}\epsilon, 1{+}\epsilon)\, \hat{A}_i \right) \right],
\end{align}
where $r_{i,t} = \pi_\theta(y_{i,t} | x, y_{i,<t}) / \pi_{\theta_{\mathrm{old}}}(y_{i,t} | x, y_{i,<t})$ is the importance sampling ratio and $\hat{A}_i = (R(x, y_i) - \mu_G) / \sigma_G$ is the group-normalized advantage.

RL training for LLMs typically operates in a loop of three stages: (1) \textit{rollout} uses the current policy to generate trajectories in an inference engine, potentially involving multi-turn sandbox or tool interactions; (2) \textit{reward} evaluates these generated trajectories with a reward model or verifier; and (3) \textit{update} optimizes the policy inside a training engine using policy gradient methods.
The asynchronous RL or partial rollout frameworks are commonly adopted to mitigate the bubble overhead caused by long-tail trajectories during rollout~\citep{fu2025areal,wang2025roll,slime2025,qin2025seer,minimax2026forge}.
Despite asynchronous designs, the rollout stage remains the dominant computational bottleneck. While MTP offers a powerful acceleration paradigm to alleviate this burden, its direct application in RL environments exposes unique performance gaps that require further optimization.

\subsection{Degradation of MTP During RL Training}
\label{subsec:degradation}

\begin{figure}[t]
    \centering
    \includegraphics[width=\linewidth]{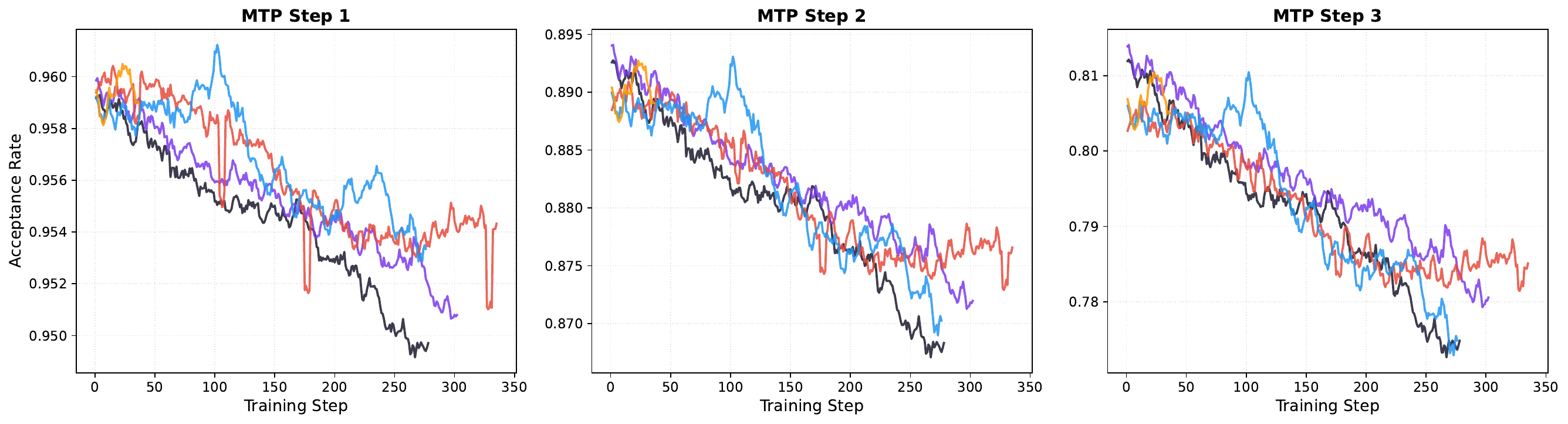}
    \caption{Per-step MTP acceptance rates during SWE-bench RL training with Qwen3.5-3.6 Plus. Each line represents a separate RL run.
    Later MTP steps exhibit progressively larger degradation: step~1 drops by 1.2\%, step~2 by 2.6\%, and step~3 by 3.5\% over the course of training.}
    \label{fig:mtp_position_acc}
\end{figure}

During RL training, MTP acceptance rates degrade significantly across prediction steps. As shown in Fig.~\ref{fig:mtp_position_acc}, later steps experience progressively larger drops. The per-step acceptance rate decline ranges from 1.2\% at step~1 to 3.5\% at step~3.

Recent work~\citep{minimax2026forge,chen2025respec,li2025mtp} primarily attributes this degradation to distribution mismatch. Specifically, a gap emerges between the static draft predictions $q = q_\phi(\cdot \mid x, y_{<t})$ and the evolving target distribution $p = \pi_\theta(\cdot \mid x, y_{<t})$ because backbone weight updates leave the draft heads behind. While this mismatch exists, we argue that this perspective is incomplete. We identify shifts in the target model's entropy $\mathcal{H}(p)$ during RL training as another fundamental driver. These entropy shifts inherently alter the achievable acceptance bounds regardless of draft accuracy. These two factors compound through the multi-step acceptance structure:
\begin{enumerate}
    \item \textbf{Single-step degradation}: 
    The per-token acceptance rate $\alpha_i$ continuously decreases as the TV distance $d_{\mathrm{TV}}(p, q)$ grows, driven by the persistent divergence between the draft-target distribution.
    \item \textbf{Multi-step compounding}: For $\gamma$-step MTP, the expected acceptance length involves products of per-step acceptance rates, so degradation compounds multiplicatively:
    $\mathbb{E}[L] = \sum_{j=1}^{\gamma} \prod_{i=1}^{j} \alpha_i$.
\end{enumerate}

Crucially, our decomposition analysis in \S\ref{sec:formulation} and Fig.~\ref{fig:decomposition_delta} challenges the conventional mismatch-centric view. We demonstrate that the entropy-driven component actually dominates the acceptance rate fluctuation during RL training. The distribution mismatch component remains comparatively small. This key insight reshapes our understanding of MTP degradation and directly motivates our subsequent optimization strategy.

\section{Target Entropy Constraints on MTP Acceptance}
\label{sec:formulation}

In this section, we analyze how the target model's entropy fundamentally constrains MTP acceptance rates, which explains the acceptance rate degradation driven by entropy shifts during RL training. This further motivates our training objectives in \S\ref{sec:mtp_training}.

\subsection{Formulation}

Consider a fixed position $t$ in the generation process.
Let $p \in \Delta^{|\mathcal{V}|}$ denote the target model's next-token distribution and $q \in \Delta^{|\mathcal{V}|}$ denote the draft head's distribution, where $\mathcal{V}$ is the vocabulary.
We define the \textit{target entropy} as:
\begin{align}
\mathcal{H}(p) = -\sum_{v \in \mathcal{V}} p(v) \log p(v),
\end{align}
which measures the uncertainty of the target model's prediction.
A low entropy indicates a confident, peaked distribution, while a high entropy indicates a spread-out distribution.

We are interested in understanding how $\mathcal{H}(p)$ constrains the achievable acceptance rate $\alpha^{\mathrm{TO}}$ and $\alpha^{\mathrm{RS}}$ defined in Eq.~\eqref{equ:target_only} and~\eqref{equ:reject_sample}.

\subsection{MTP with Target-Only Sampling}

Under target-only sampling, the acceptance rate depends on how well the draft's greedy prediction $\hat{y} = \arg\max_y q(y)$ aligns with the target's high-probability region.
When the target entropy $\mathcal{H}(p)$ is low (i.e., $p$ is peaked on a few tokens), even a moderately accurate draft model can achieve high acceptance by placing mass on the dominant tokens.
Conversely, when $\mathcal{H}(p)$ is high, the target distribution spreads over many tokens, reducing $\max_y p(y)$ and increasing the probability of ranking errors.

\begin{proposition}[Entropy-Dependent Acceptance under Target-Only Sampling]
\label{prop:to_entropy}
For a well-trained draft model, $\alpha^{\mathrm{TO}} = \max_y p(y)$, which is a monotonically decreasing function of $\mathcal{H}(p)$, lower-bounded by $\exp(-\mathcal{H}(p))$, and empirically well-approximated as linear (Fig.~\ref{sfig:entropy_vs_accept_length}):
\begin{align}
\label{equ:to_linear}
\alpha^{\mathrm{TO}} \approx a^{\mathrm{TO}} - b^{\mathrm{TO}} \cdot \mathcal{H}(p),
\end{align}
with positive constants $a^{\mathrm{TO}}, b^{\mathrm{TO}}$.
Ranking errors under imperfect drafts steepen the slope but preserve linearity (\S\ref{subsec:to_analysis}).
\end{proposition}
\begin{proof}[Proof sketch]
When the draft correctly identifies the target's top-1 token ($\arg\max q = \arg\max p$), the acceptance rate reduces to $\alpha^{\mathrm{TO}} = \max_y p(y)$.
By Jensen's inequality applied to the concave logarithm, $\log(\max_y p(y)) \geq -\mathcal{H}(p)$, i.e., $\max_y p(y) \geq \exp(-\mathcal{H}(p))$.
Writing $\alpha^{\mathrm{TO}} = f(\mathcal{H})$ for some smooth decreasing function $f$ and performing a first-order Taylor expansion around a reference entropy $\bar{\mathcal{H}}$:
\begin{align}
\alpha^{\mathrm{TO}} \approx \underbrace{\bigl[f(\bar{\mathcal{H}}) - f'(\bar{\mathcal{H}})\bar{\mathcal{H}}\bigr]}_{a^{\mathrm{TO}}} + \underbrace{f'(\bar{\mathcal{H}})}_{-b^{\mathrm{TO}}} \cdot \mathcal{H}(p).
\end{align}
Since $f$ is decreasing, $b^{\mathrm{TO}} = -f'(\bar{\mathcal{H}}) > 0$.
See \S\ref{subsec:to_analysis} for the full derivation including imperfect draft corrections.
\end{proof}

This linear relationship is remarkably robust across different model sizes, tasks, and training stages, as shown in Fig.~\ref{sfig:entropy_vs_accept_length}.

\subsection{MTP with Rejection Sampling}
\label{sec:mtp-with-rs}

Under rejection sampling, the acceptance rate equals the TV overlap between $p$ and $q$ (Eq.~\eqref{equ:reject_sample}).
We can decompose the TV distance using the identity $|a - b| = a + b - 2\min(a, b)$ and probability normalization:
\begin{align}
d_{\mathrm{TV}}(p, q) = \frac{1}{2}\sum_{v} \bigl(p(v) + q(v) - 2\min(p(v), q(v))\bigr) = 1 - \sum_{v} \min\bigl(p(v), q(v)\bigr).
\end{align}
Therefore, maximizing the acceptance rate is equivalent to minimizing the TV distance:
\begin{align}
\label{equ:rs_tv}
\alpha^{\mathrm{RS}} = 1 - d_{\mathrm{TV}}(p, q).
\end{align}

As a result, the acceptance rate is no longer bounded by the policy's entropy directly. However, empirical results show that the connection to entropy remains after switching to rejection sampling. In our further investigation, we find that under CE/KL-trained draft models, even small per-token mismatches accumulate when $p$ has high entropy, leading to a larger TV distance.
This motivates our deeper analysis of how the training objective affects this relationship as follows.

\begin{proposition}[Entropy-Dependent Acceptance under CE/KL-Trained Rejection Sampling]
\label{prop:rs_linear}
Under CE/KL-trained draft models, the rejection sampling acceptance rate satisfies:
\begin{align}
\label{equ:rs_linear}
\alpha^{\mathrm{RS}} \approx a^{\mathrm{RS}} - b^{\mathrm{RS}} \cdot \mathcal{H}(p),
\end{align}
with positive constants $a^{\mathrm{RS}}, b^{\mathrm{RS}}$, where $b^{\mathrm{RS}}$ is comparable to $b^{\mathrm{TO}}$ though empirically slightly steeper (\S\ref{subsec:ce_rs}, Fig.~\ref{fig:entropy_vs_accept}).
\end{proposition}
\begin{proof}[Proof sketch]
The CE/KL gradient $q_j - p_j$ produces uniform per-token mismatch $|\eta_v| \lesssim \sigma$.
Since the effective support size scales as $|\mathcal{S}_{\mathrm{eff}}| \approx \exp(\mathcal{H}(p))$, the TV distance accumulates as $d_{\mathrm{TV}} \approx \frac{\sigma}{2}\exp(\mathcal{H}(p))$, yielding $\alpha^{\mathrm{RS}} \approx 1 - \frac{\sigma}{2}\exp(\mathcal{H}(p))$.
Linearizing the exponential over the operating entropy range gives the stated form.
See \S\ref{subsec:ce_rs} for details.
\end{proof}

Therefore, under CE/KL-trained MTP, both rejection and target-only sampling remain sensitive to entropy shifts. As policy entropy fluctuates significantly during RL training, this sensitivity inherently limits the achievable speedup.
\section{Optimizing MTP for RL Training}
\label{sec:mtp_training}

    As discussed above, MTP acceptance rates can degrade significantly during RL training due to the entropy bound. In this section, we develop the novel end-to-end TV loss to address this challenge.

\subsection{TV Loss: Directly Optimizing Acceptance Rate}
\label{subsec:tv_loss}

\paragraph{Motivation.}
Conventional MTP training minimizes the cross-entropy (CE) loss or the KL divergence between the target and draft distributions.\footnote{Throughout this paper, ``KL'' refers to the forward KL divergence $D_{\mathrm{KL}}(p\|q)$ unless otherwise noted. CE and forward KL differ only by the constant $\mathcal{H}(p)$ and yield identical gradients. We analyze the reverse KL divergence $D_{\mathrm{KL}}(q\|p)$ separately in \S\ref{app:reverse_kl}.}
However, the rejection sampling acceptance rate is determined by the TV distance (Eq.~\eqref{equ:rs_tv}), not the KL divergence.
By Pinsker's inequality, $d_{\mathrm{TV}}(p,q) \leq \sqrt{D_{\mathrm{KL}}(p\|q)/2}$, so KL provides only an indirect upper bound, and minimizing it does not efficiently minimize TV distance.
This motivates directly optimizing the TV distance as the MTP training objective.

\paragraph{TV Loss.}
We propose to directly minimize the TV distance:
\begin{align}
\label{equ:tv_loss}
\mathcal{L}_{\mathrm{TV}} = d_{\mathrm{TV}}(p, q) = 1 - \sum_{v \in \mathcal{V}} \min\bigl(p(v), q(v)\bigr),
\end{align}
where $p$ is treated as a constant (detached from the computation graph) and gradients flow only through $q$.

\paragraph{Gradient Analysis.}
Let the draft head output logits $z \in \mathbb{R}^{|\mathcal{V}|}$ with $q_j = \mathrm{softmax}(z)_j$.
The gradient of the TV loss with respect to $z_j$ is:
\begin{align}
\label{equ:tv_grad}
\frac{\partial \mathcal{L}_{\mathrm{TV}}}{\partial z_j} = -q_j \Bigl[ \mathbbm{1}[q_j \leq p_j] - S \Bigr], \quad \text{where} \quad S = \sum_{v} \mathbbm{1}[q_v \leq p_v] \cdot q_v.
\end{align}

\begin{proposition}[Bounded Gradient]
\label{prop:bounded_grad}
The TV loss gradient is bounded: $\left| \frac{\partial \mathcal{L}_{\mathrm{TV}}}{\partial z_j} \right| \leq 1$ for all $j$.
\end{proposition}
\begin{proof}
Since $q_j \in [0, 1]$ and $|\mathbbm{1}[q_j \leq p_j] - S| \leq 1$ (as both the indicator and $S \in [0, 1]$), we have $\left| \frac{\partial \mathcal{L}_{\mathrm{TV}}}{\partial z_j} \right| = q_j \cdot |\mathbbm{1}[q_j \leq p_j] - S| \leq 1$.
\end{proof}

This bounded gradient property ensures training stability, in contrast to KL divergence whose gradient $\frac{\partial D_{\mathrm{KL}}}{\partial z_j} = q_j - p_j$ can exhibit large magnitudes when $q$ and $p$ disagree significantly.

\paragraph{Intuitive Interpretation.}
The TV loss gradient has a natural interpretation in terms of the rejection sampling mechanism:
\begin{itemize}[leftmargin=8mm, itemsep=1mm]
    \item For tokens where $q_j \leq p_j$ (tokens that would be accepted): the gradient increases the logit, encouraging the draft to assign more mass.
    \item For tokens where $q_j > p_j$ (tokens that would be rejected): the gradient decreases the logit, suppressing overconfident predictions.
    \item For tokens where $q_j \approx 0$ (irrelevant tokens): the gradient is automatically $\approx 0$ (since it is proportional to $q_j$), avoiding wasted optimization effort on the long tail of the vocabulary.
\end{itemize}
This selective gradient behavior contrasts with KL divergence, which applies gradients to all tokens regardless of their relevance to the acceptance decision.

\paragraph{Comparison of CE, KL, and TV Gradients.}
Table~\ref{tab:grad_comparison} summarizes the gradient structures of the three training objectives.
The key distinction lies in whether the gradient is proportional to $q_j$: CE loss produces uniform per-token mismatch ($q_j - p_j$) that distributes optimization effort uniformly across the vocabulary, including irrelevant low-probability tokens.
In contrast, both reverse KL and TV loss exhibit $q_j$-proportional gradients with natural tail suppression, concentrating updates on tokens the draft already assigns non-negligible mass.
However, despite this shared property, reverse KL yields negligible acceptance rate improvements over CE (\S\ref{sec:exp}), because its zero-forcing behavior allows the draft to drop modes of $p$ and its asymmetric penalty drives $q \leq p$ globally---both reducing the TV overlap $\sum_v \min(p, q)$ (see \S\ref{app:reverse_kl} for a detailed analysis).
TV loss avoids these pitfalls by directly optimizing the acceptance-relevant quantity and producing a probability-proportional mismatch that decouples acceptance from target entropy.

\begin{table}[t]
    \centering
    \caption{Gradient comparison across training objectives. $C$ denotes a global constant ($S$ for TV, $D_{\mathrm{KL}}(q\|p)$ for reverse KL). See \S\ref{app:tv_grad}-\S\ref{app:reverse_kl} for derivations.}
    \label{tab:grad_comparison}
    \begin{tabular}{lccc}
        \toprule
        \textbf{Property} & \textbf{Forward KL} & \textbf{Reverse KL} & \textbf{TV Loss} \\
        \midrule
        Gradient & $q_j - p_j$ & $q_j[\log(q_j/p_j) - C]$ & $-q_j[\mathbbm{1}[q_j \leq p_j] - C]$ \\
        $\propto q_j$? & No & Yes & Yes \\
        Tail suppression & No & Yes & Yes \\
        \bottomrule
    \end{tabular}
\end{table}

\subsection{End-to-End Multi-Step TV Loss}
\label{subsec:e2e_tv}

For $\gamma$-step MTP, the expected acceptance length is:
\begin{align}
\mathbb{E}[L] = \sum_{j=1}^{\gamma} \prod_{i=1}^{j} \alpha_i = \alpha_1 + \alpha_1 \alpha_2 + \alpha_1 \alpha_2 \alpha_3 + \cdots + \prod_{i=1}^{\gamma} \alpha_i,
\end{align}
where $\alpha_i = 1 - d_{\mathrm{TV}}(p_i, q_i)$ is the per-step acceptance rate at step $i$.
Directly optimizing the \textit{average} per-step TV distance $\frac{1}{\gamma}\sum_{i=1}^{\gamma} d_{\mathrm{TV}}(p_i, q_i)$ does not account for the multiplicative structure of multi-step acceptance.
We therefore propose the \textit{end-to-end (e2e) TV loss}:
\begin{align}
\label{equ:e2e_tv}
\mathcal{L}_{\mathrm{e2e}} = 1 - \frac{1}{\gamma} \sum_{j=1}^{\gamma} \prod_{i=1}^{j} \alpha_i = 1 - \frac{1}{\gamma} \sum_{j=1}^{\gamma} \prod_{i=1}^{j} \bigl(1 - d_{\mathrm{TV}}(p_i, q_i)\bigr).
\end{align}
This loss directly optimizes the normalized expected acceptance length, naturally weighting earlier steps more heavily (since they appear in more product terms) and capturing the compounding effect of multi-step verification.
This can be regarded as a dynamic step-wise weighting scheme: since $\alpha_i$ depends on the current draft quality, the effective weight of each position adapts automatically as training progresses, shifting emphasis toward steps that currently limit acceptance.
This contrasts with prior work that uses fixed position-dependent weights, such as head-dependent loss weights \citep{medusa-cai-2024,li2025eagle3}, exponentially decaying block-position weights \citep{chen2026dflash}, fixed decay on rejected positions \citep{lei2026draftopd}, or per-position weights on a CE base \citep{wu2026dpace}.
\subsection{Impact of Training Objective on Entropy-Acceptance Relationship}
\label{subsec:training_impact}

Having introduced the TV loss, we now analyze why it fundamentally outperforms CE/KL training in the context of RL, where the target entropy shifts continuously.
The linear relationships in Eq.~\eqref{equ:to_linear} and~\eqref{equ:rs_linear} characterize draft models trained with CE/KL loss; we show that the choice of training objective fundamentally alters the entropy-acceptance relationship.
The full derivation is provided in \S\ref{app:entropy_accept}; here we state the main results.

\paragraph{Pinsker's inequality and the KL--TV gap.}
By Pinsker's inequality:
\begin{align}
\label{equ:pinsker}
d_{\mathrm{TV}}(p, q) \leq \sqrt{\frac{1}{2} D_{\mathrm{KL}}(p \| q)},
\end{align}
$\sqrt{D_{\mathrm{KL}}/2}$ provides only an upper bound on $d_{\mathrm{TV}}$, and KL optimization allocates model capacity inefficiently for minimizing TV distance:
\textit{Minimizing the KL divergence does not efficiently minimize the TV distance}, which is the quantity that directly determines the rejection sampling acceptance rate.

\paragraph{CE/KL Training: Uniform Mismatch.}
The KL divergence gradient $\frac{\partial D_{\mathrm{KL}}}{\partial z_j} = q_j - p_j$ applies optimization pressure proportional to the absolute difference $|q_j - p_j|$, regardless of the magnitude of $p_j$ relative to other tokens.
Under a capacity-limited draft model, this produces approximately uniform per-token mismatch: $|q^*(v) - p(v)| \lesssim \sigma$ for a constant $\sigma$.
As shown in Proposition~\ref{prop:rs_linear}, this uniform mismatch accumulates over the effective support $|\mathcal{S}_{\mathrm{eff}}| \approx \exp(\mathcal{H}(p))$, yielding an entropy-dependent acceptance rate.

\paragraph{TV Training: Probability-Proportional Mismatch.}
The TV loss gradient (Eq.~\eqref{equ:tv_grad}) is proportional to $q_j$, concentrating optimization on high-probability tokens and automatically ignoring the long tail.
Under a capacity-limited draft model, each token receives optimization resources proportional to its probability $q_j \approx p_j$, so the per-token mismatch also scales with $p(v)$ rather than remaining at a uniform level.
This produces probability-proportional mismatch: $|q^*(v) - p(v)| \lesssim \delta \cdot p(v)$ for a constant $\delta$ (see \S\ref{subsec:tv_rs} for a detailed derivation).

\begin{proposition}[Reduced Entropy Dependence under TV Training]
\label{prop:tvd_entropy}
When the per-token mismatch satisfies $|q^*(v) - p(v)| \lesssim \delta \cdot p(v)$, the TV distance is bounded independently of entropy:
\begin{align}
\label{equ:tvd_entropy}
d_{\mathrm{TV}}(p, q^*_{\mathrm{TV}}) \leq \frac{\delta}{2} \sum_v p(v) = \frac{\delta}{2},
\end{align}
yielding $\alpha^{\mathrm{RS}}_{\mathrm{TV}} \geq 1 - \delta/2$.
In practice, the draft head has finite capacity, so $\delta$ may exhibit weak entropy dependence $\delta = \delta(\mathcal{H})$, but empirically the entropy--acceptance slope is reduced by over $95\%$ compared to CE/KL training (Fig.~\ref{fig:entropy_vs_accept}).
\end{proposition}
\begin{proof}[Proof sketch]
The TV gradient is proportional to $q_j$ (Eq.~\eqref{equ:tv_grad}), so each token's optimization resource scales with its probability, producing $|q^*(v) - p(v)| \lesssim \delta \cdot p(v)$ (\S\ref{subsec:tv_rs}).
Summing: $d_{\mathrm{TV}} = \frac{1}{2}\sum_v |q^* - p| \leq \frac{\delta}{2}\sum_v p(v) = \frac{\delta}{2}$, which is entropy-independent since $\sum_v p(v) = 1$.
\end{proof}

This analysis explains the empirical observation that TV-trained draft models achieve substantially more stable acceptance rates across varying target entropy, while CE/KL-trained models exhibit a strong negative correlation (Fig.~\ref{fig:entropy_vs_accept}).

\section{MTP Adaptation Strategy for RL}
\label{sec:mtp_adapt}

A key question for using MTP in RL pipelines is whether we need online updates of the MTP module during RL training.
We investigate this through a decomposition analysis that disentangles the two factors driving acceptance rate changes.

\subsection{Decomposition: Entropy vs.\ Mismatch in RL}
\label{subsec:decomposition}

Using the linear entropy--acceptance relationship established in \S\ref{sec:formulation}, we decompose the change in acceptance length during RL training as:
\begin{align}
\label{equ:decomposition}
\Delta\alpha_t = \underbrace{b \cdot (\mathcal{H}_t - \mathcal{H}_0)}_{\Delta\alpha_{\mathrm{entropy}}} + \underbrace{\Delta\alpha_t - b \cdot (\mathcal{H}_t - \mathcal{H}_0)}_{\Delta\alpha_{\mathrm{mismatch}}},
\end{align}
where $b$ is the entropy--acceptance slope estimated from the early phase of each experiment, $\mathcal{H}_0$ is the initial entropy, and $\Delta\alpha_t = \alpha_t - \alpha_0$ is the total acceptance change at step $t$.
The first term captures the acceptance change attributable to entropy shifts alone (assuming a fixed draft--target relationship), while the residual captures the effect of growing draft--target mismatch due to backbone weight updates.

\begin{figure*}[t]
    \centering
    \includegraphics[width=\linewidth]{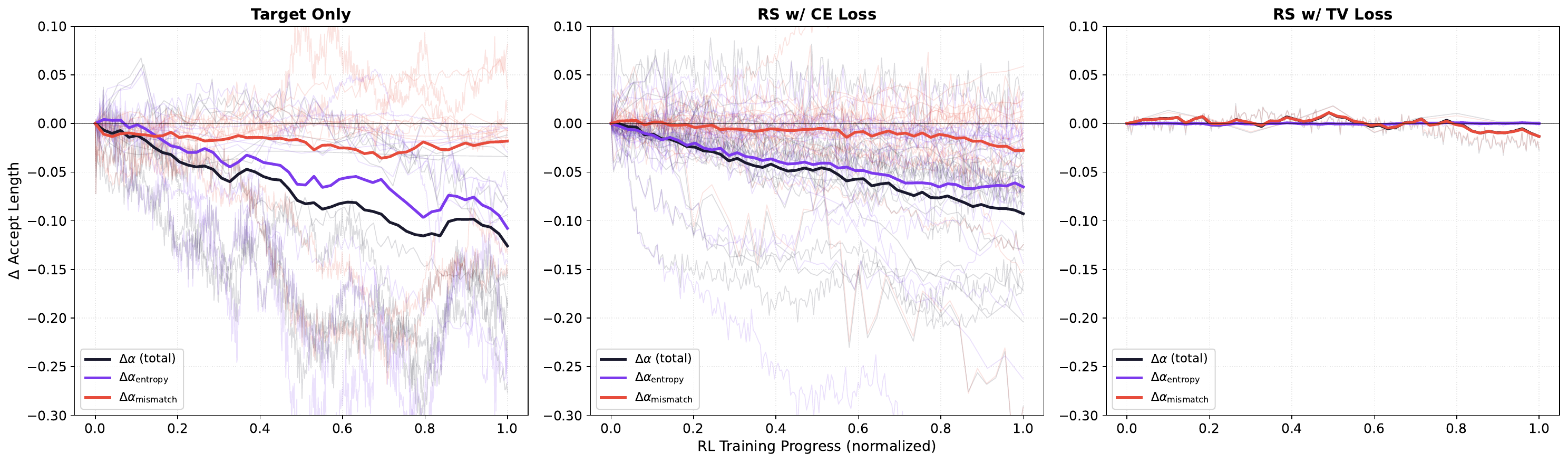}
    \caption{
        Decomposition of acceptance length changes during RL training.
        $\Delta\alpha$ (total, gray) is decomposed into an entropy-driven component $\Delta\alpha_{\mathrm{entropy}} = b \cdot (\mathcal{H}_t - \mathcal{H}_0)$ (orange) and a draft--target mismatch component $\Delta\alpha_{\mathrm{mismatch}}$ (green).
        Under target-only sampling, both entropy increase and growing mismatch contribute to acceptance degradation.
        Under rejection sampling with CE loss, the degradation is almost entirely entropy-driven, with mismatch remaining near zero.
        RS with TV loss shows near-zero change across all components, confirming the stability of TV-trained drafts.
    }
    \label{fig:decomposition_delta}
\end{figure*}

As shown in Fig.~\ref{fig:decomposition_delta}:
\textbf{(1)} Under target-only sampling, both entropy increase and growing mismatch contribute to acceptance degradation, as the greedy draft prediction becomes increasingly misaligned with the evolving target.
\textbf{(2)} Under rejection sampling with CE loss, the degradation is almost entirely entropy-driven ($\Delta\alpha_{\mathrm{mismatch}} \approx 0$), indicating that RL weight updates do not significantly affect the draft--target TV overlap.
\textbf{(3)} Under rejection sampling with TV loss, near-zero change is observed across all components, confirming that TV-trained drafts are robust to both entropy shifts and weight updates.

\subsection{Pre-RL Adaptation is Sufficient}
\label{subsec:pre_rl}

The decomposition analysis leads to a key practical insight: since the draft--target mismatch induced by RL weight updates is negligible under rejection sampling, \textit{updating the MTP heads during RL is unnecessary}.
A one-time pre-RL adaptation with TV loss---applied during the SFT stage before RL begins---is sufficient to produce draft models that maintain high acceptance rates throughout RL training (Fig.~\ref{fig:rl_accept_len}).
This eliminates the memory overhead of maintaining MTP optimizer states and the computational cost of MTP gradient updates during RL.

Empirically, as shown in Fig.~\ref{fig:rl_mtp_train_accept_len}, continuing to update MTP weights during RL yields no significant improvement when starting from a well-trained TV checkpoint.
Worse, updating with CE loss during RL causes the acceptance rate to degrade toward the RS w/ CE baseline, as CE loss makes the draft distribution smoother and erodes the gains from TV training (\S\ref{subsec:distribution_shift_in_rl_training}).

\subsection{Cross Training of MTP and Backbone}
\label{subsec:cross_training}

When MTP co-training during RL is desired (e.g., for target-only sampling where mismatch is non-negligible), we find that joint training with separate learning rates and separate gradient norm normalization provides the best trade-off.
The backbone gradients are not affected by the MTP loss (which only flows through the draft heads), ensuring that the MTP training does not interfere with the RL optimization of the backbone.

\section{Experiments}
\label{sec:exp}

We validate the effectiveness of our method \textit{Bebop} through three sets of experiments:
(1) the impact of different multi-step MTP loss objectives on acceptance rate during SFT;
(2) the benefits of e2e TV loss with rejection sampling on acceptance rate, speedup, and training stability during RL;
and (3) the gains from updating MTP parameters during the RL stage.

\subsection{Multi-Step MTP Training Improves Acceptance Rate}

We first evaluate how different loss objectives affect MTP acceptance rates during the SFT stage. Specifically, we compare four MTP training objectives:
\begin{enumerate}[label=(\arabic*), leftmargin=8mm, itemsep=1mm]
    \item \textbf{CE loss}: standard cross-entropy between draft and target distributions;
    \item \textbf{KL loss}: KL divergence $D_{\mathrm{KL}}(p \| q)$;
    \item \textbf{Reverse KL loss}: Reverse KL divergence $D_{\mathrm{KL}}(q \| p)$ (Eq.~\eqref{eq:rkl_gradient});
    \item \textbf{TV loss}: per-step TV distance (Eq.~\eqref{equ:tv_loss});
    \item \textbf{e2e TV loss}: end-to-end multi-step TV loss (Eq.~\eqref{equ:e2e_tv}).
\end{enumerate}

We conduct the primary experiments on Qwen3.5-35A3B~\citep{qwen3.5} using mixed RFT data. All experiments use a constant learning rate of $3.5 \times 10^{-5}$ with 3\% warmup steps, training for 1 epoch with Megatron~\citep{shoeybi2020megatronlm} at a global batch size of 256 and a sequence length of 256K.
During multi-step MTP training, we perform forward and backward passes over 5 MTP steps while freezing the LLM backbone.
All evaluations use $\gamma = 3$ (i.e., the target model verifies 4 tokens at a time).
We further extend our experiments to Qwen3.6-35A3B, Qwen3.6-Plus, and Qwen3.7-Plus, training on different data mixtures including domain-specific data (code, agent, reasoning) and mixed RFT data. The throughput is measured using SGLang's MTP implementation with rejection sampling (see \S\ref{app:rs_inference} for implementation details).

\paragraph{Rejection Sampling Acceptance.}
Table~\ref{tab:accept_rate} reports the acceptance rate improvements of our proposed e2e TV loss compared to the CE and KL baselines on Qwen3.5-35A3B.
Across all tasks, e2e TV loss consistently improves rejection sampling acceptance rates by 3--8\% on in-distribution tasks (Math, Code, Agent, SWE) and up to 2.3\% on the out-of-distribution MT-Bench~\citep{zheng2023mtbench} task.
Notably, on Agent tasks where the CE baseline already achieves a high acceptance rate of 90.3\%, e2e TV loss further pushes it to 97.0\%, a level that substantially improves rollout efficiency in both RL training and agentic inference.

Beyond the primary experiments, we evaluate across a broader set of models and data configurations.
As shown in Fig.~\ref{fig:accept_rate_combined}, we train Qwen3.6-35A3B, Qwen3.6-Plus, and Qwen3.7-Plus on different data mixtures and track per-step acceptance rates throughout training.
Several patterns emerge.
First, CE loss causes a pronounced and persistent decline in Step~1 acceptance rate during training, as it distributes optimization effort across the entire vocabulary. In contrast, TV loss maintains stable or slightly improving Step~1 acceptance.
Second, the advantage of e2e TV loss becomes increasingly prominent at later MTP steps: at Step~3, TV loss outperforms CE loss by approximately 5\%, while at Step~2 the margin is 2.5--5\%.
Third, the gains are task-dependent: agentic tasks benefit the most, with improvements up to 8\% on Agent and SWE-Bench~\citep{jimenez2024swebench}, while reasoning and conversational tasks see gains of 4--5\%.
Finally, MTP acceptance rates exhibit strong generalization. Models trained entirely without agent-specific data still achieve approximately 70\% acceptance on agent tasks. Specifically, TV loss provides larger improvements on in-distribution domains than on out-of-distribution tasks.

\begin{table}[t]
    \centering
    \caption{MTP acceptance rate (\%) under rejection sampling across tasks and training objectives under $\gamma=3$ on Qwen3.5-35A3B. All results are measured at convergence. $\Delta$ denotes improvement over CE loss baseline.}
    \label{tab:accept_rate}
    \begin{tabular}{lccccc}
        \toprule
        \textbf{MTP Loss} & \textbf{Math} & \textbf{Code} & \textbf{SWE} & \textbf{Agent} & \textbf{MTBench (OOD)} \\
        \midrule
        CE loss (baseline) & $75.0$ & $71.3$ & $75.1$ & $90.3$ & $65.3$ \\
        KL loss & $+0.0$ & $+0.0$ & $+0.2$ & $+0.2$ & $+0.0$ \\
        Reverse KL loss & $+1.3$ & $+1.0$ & $-0.2$ & $+1.0$ & $+0.5$ \\
        TV loss & $+2.4$ & $+2.5$ & $+3.3$ & $+5.2$ & $+1.4$ \\
        e2e TV loss (ours) & $\mathbf{+3.0}$ & $\mathbf{+3.3}$ & $\mathbf{+8.0}$ & $\mathbf{+6.7}$ & $\mathbf{+2.3}$ \\
        \bottomrule
    \end{tabular}
\end{table}

\begin{figure*}[t]
    \centering
    \subfloat[Accept length on reasoning and conversation tasks (Math, Code, MT-Bench).\label{fig:ce_vs_tv_accept_rate}]{%
        \includegraphics[width=\linewidth]{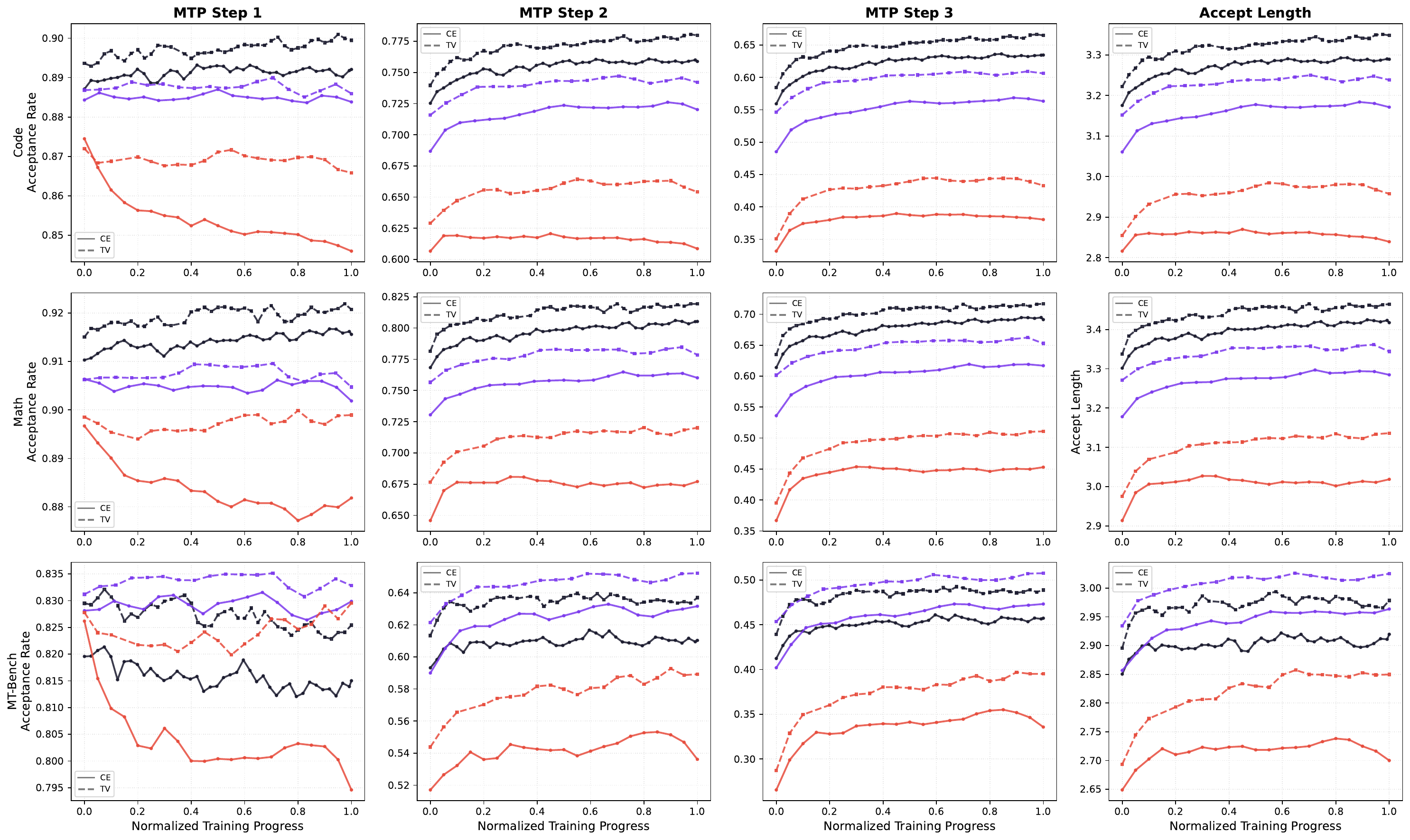}}\\
    \subfloat[Accept length on agentic and hybrid tasks (Hybrid, Agent, Long-Horizon, SWE-Bench).\label{fig:swe_vending_accept_length}]{%
        \includegraphics[width=\linewidth]{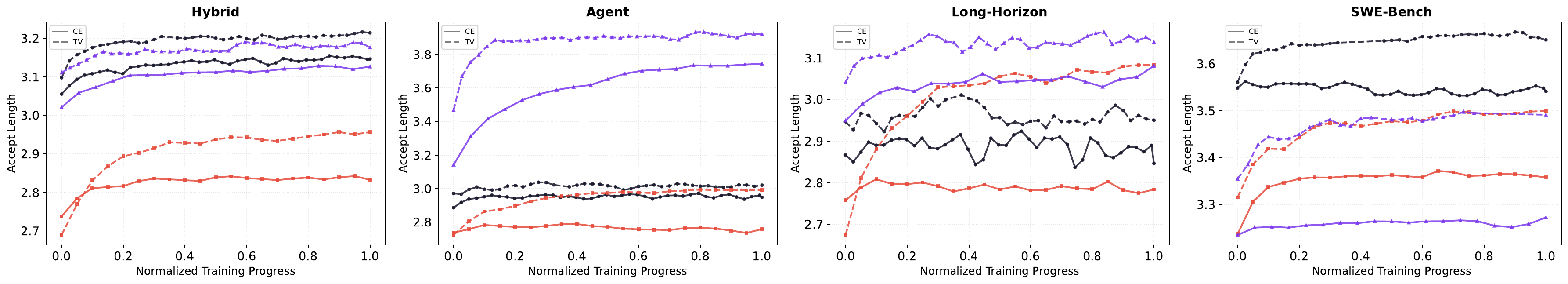}}
    \caption{CE loss (solid) vs.\ TV loss (dashed) during SFT training. TV loss consistently achieves higher acceptance rates across all MTP steps, with especially pronounced gains on agentic tasks.}
    \label{fig:accept_rate_combined}
\end{figure*}

\paragraph{Target-Only Acceptance.}
Under target-only sampling, acceptance rates are nearly identical across all training objectives ($<$0.3\% difference), as shown in Fig.~\ref{fig:to_ce_tv_accept}.
This is expected: target-only acceptance $\alpha^{\mathrm{TO}} = p(\arg\max_y q(y))$ reduces to $\max_y p(y)$ when the draft's top-1 ranking is correct, depending only on the target distribution rather than the draft's distributional shape.
In contrast, rejection sampling acceptance $\alpha^{\mathrm{RS}} = \sum_v \min(p(v), q(v))$ depends on the full distributional overlap, which is where TV loss provides its advantage.
This is consistent with our analysis in \S\ref{sec:formulation}.

\begin{figure}[t]
    \centering
    \includegraphics[width=\linewidth]{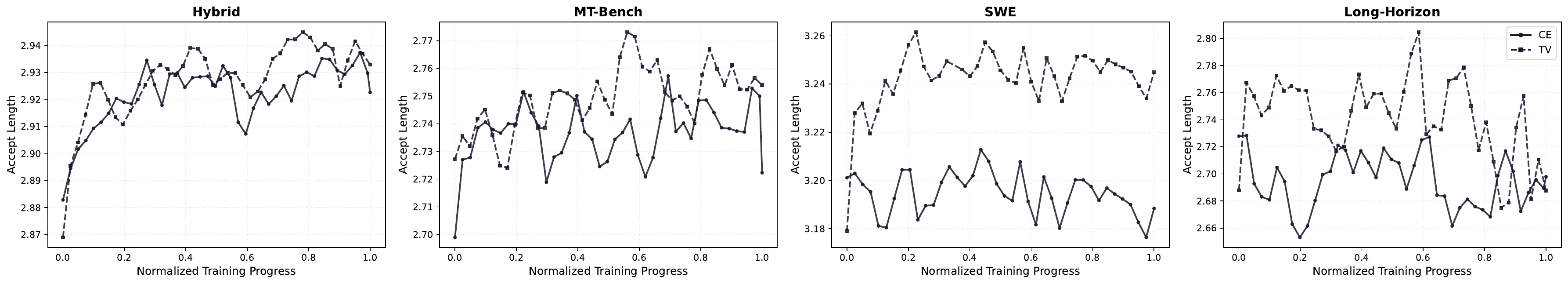}
    \caption{Accept length under target-only sampling with CE loss vs.\ TV loss during SFT training. Acceptance rates are nearly identical ($<$0.3\% difference) across all tasks, confirming that target-only acceptance depends on the target distribution rather than the draft's distributional shape.}
    \label{fig:to_ce_tv_accept}
\end{figure}

\paragraph{Throughput.}
As shown in Fig.~\ref{sfig:accept_delta_vs_throughput}, the acceptance rate improvement translates to throughput gains roughly linearly. The e2e-TV-trained Qwen3.7 Plus consistently outperforms the CE-loss-trained Qwen3.6 Plus on all datasets. These gains effectively accelerate RL rollouts, which is significant at the scale of hundreds of thousands of GPU hours.

\paragraph{Acceptance Rate Scales with Model Size.}

As shown in Table~\ref{tab:accept_rate_scale}, MTP acceptance rates after multi-step SFT training consistently increase with model size. Qwen3.7 models are trained with e2e TV loss, while Qwen3.6 models use CE loss. The acceptance rate reaches up to 95\%, especially on agent tasks, indicating that the draft model under $\gamma = 3$ has nearly converged to the backbone model.
Conversely, as model size decreases, acceptance rates degrade to varying degrees.

\subsection{TV Loss Stabilizes MTP Acceleration in RL Training}

We conduct extensive experiments in RL settings to demonstrate the effectiveness of \textit{Bebop}. We select two representative workloads spanning different generation regimes:
\begin{enumerate}[label=(\arabic*), leftmargin=8mm, itemsep=1mm]
    \item \textbf{Reasoning RL}: long chain-of-thought tasks including math reasoning, code reasoning, and instruction-following, with a maximum generation length of 64K tokens. Evaluation benchmarks: HMMT25~\citep{dekoninck2026matharena}, AIME25~\citep{aime25}, and LiveCodeBench~\citep{jain2024livecodebench}.
    \item \textbf{SWE RL}: multi-turn code editing tasks where each turn involves thinking, tool calling, and tool execution, with tool responses appended to the previous context. Maximum generation length is 128K tokens with up to 200 turns. Evaluation benchmark: SWE-Verified~\citep{jimenez2024swebench}.
\end{enumerate}

For all RL experiments, we use SGLang~\citep{sglang} as the rollout engine within an asynchronous RL framework built on top of veRL~\citep{verl}, with a learning rate of $1 \times 10^{-6}$ or $2 \times 10^{-6}$.

\begin{figure*}[t]
    \centering
    \subfloat[Reasoning RL.\label{fig:rl_hybrid_accept_len}]{%
        \includegraphics[width=0.33\linewidth]{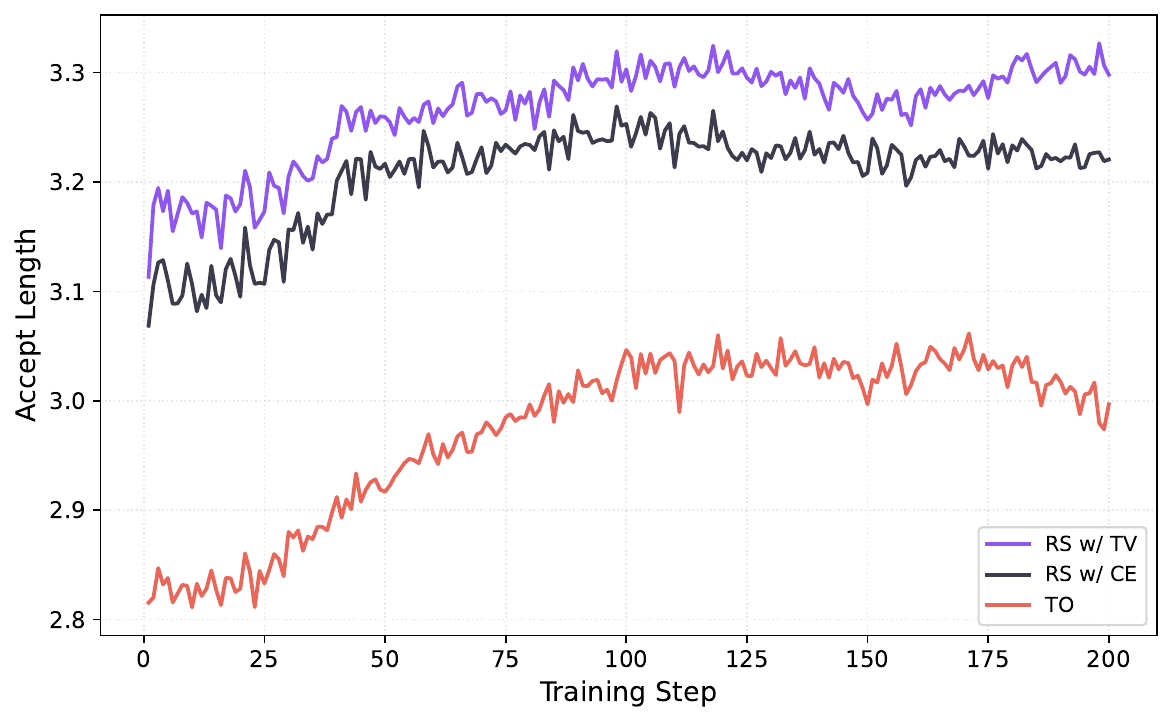}}
    \subfloat[SWE RL.\label{fig:rl_swe_accept_len}]{%
        \includegraphics[width=0.33\linewidth]{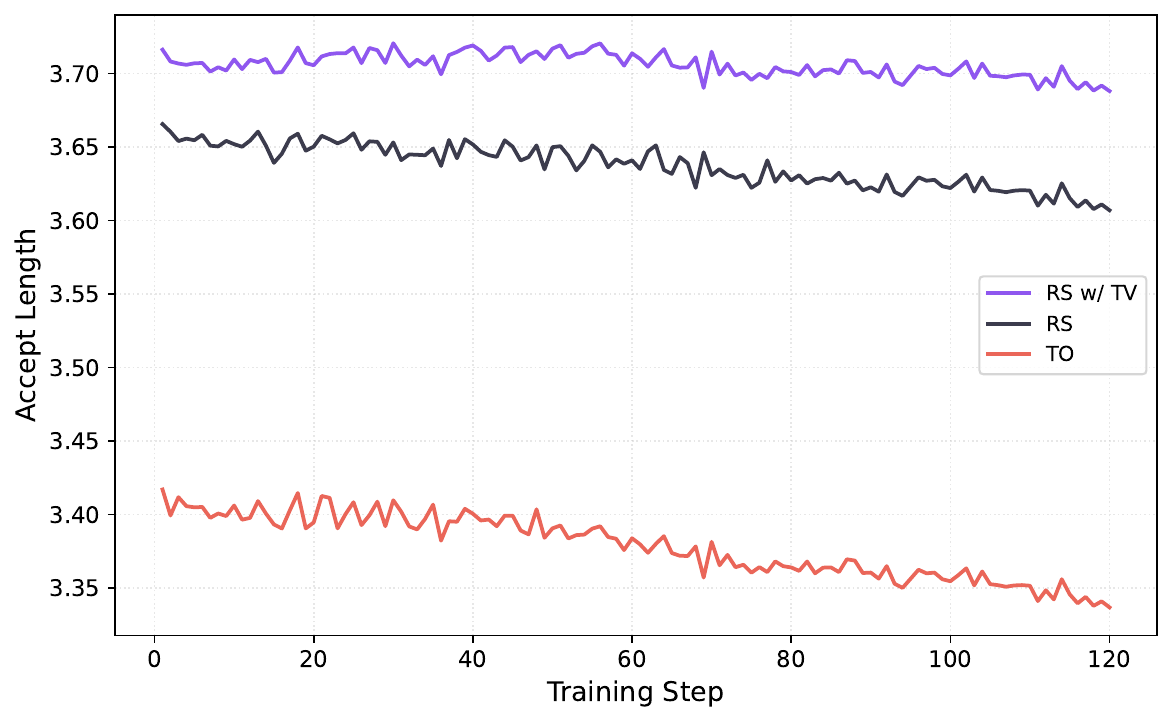}}
    \subfloat[SWE RL in Qwen-3.7 Max.\label{fig:rl_swe_max_accept_len}]{%
        \includegraphics[width=0.33\linewidth]{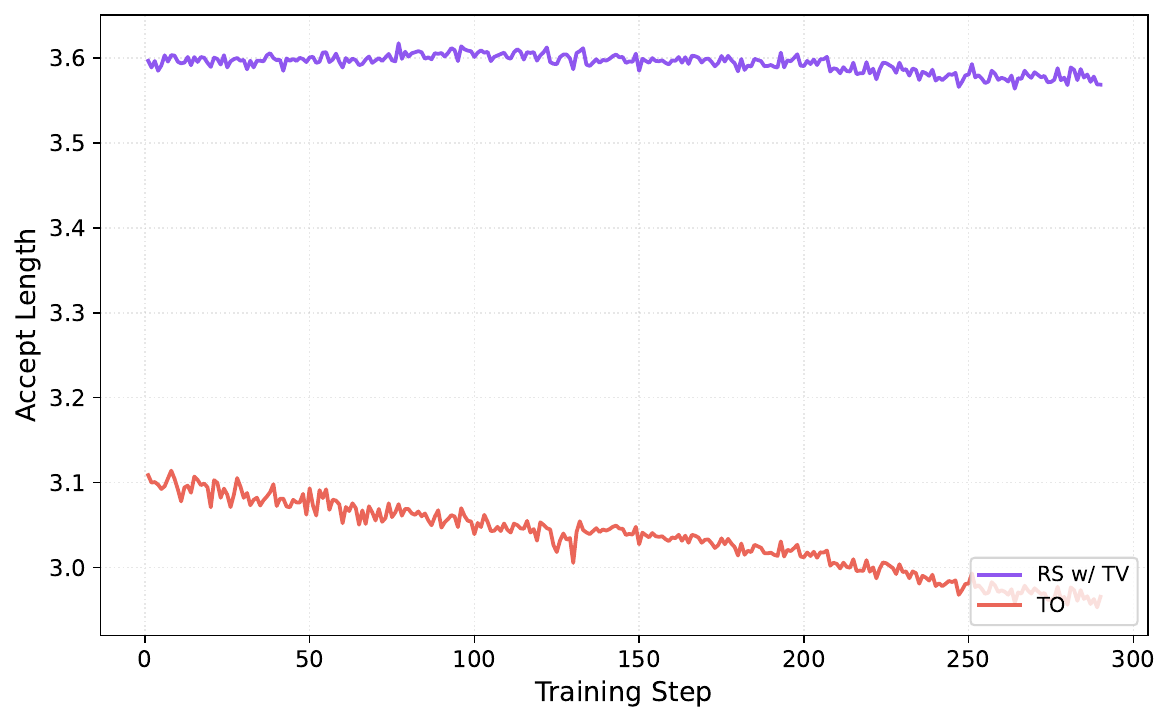}}
    \caption{Accept length during RL training across different workloads in Qwen3.6-Plus and Qwen3.7-Max. Rejection sampling with TV loss (RS w/ TV) consistently maintains higher accept lengths compared to target-only (TO) and rejection sampling with CE loss (RS w/ CE).}
    \label{fig:rl_accept_len}
\end{figure*}

\begin{table}[t]
    \centering
    \caption{MTP acceptance rate (\%) under rejection sampling across tasks and training objectives under $\gamma=3$ on different models. Qwen3.7 models are trained with e2e TV loss; all others are trained with CE loss.}
    \label{tab:accept_rate_scale}
    \begin{tabular}{lccccccc}
        \toprule
        \textbf{Model} & \textbf{Math} & \textbf{Code} & \textbf{Hybrid} & \textbf{SWE} & \textbf{Agent}  & \textbf{Long-horizon} & \textbf{MTBench} \\
        \midrule
        Qwen3.7-Max & $87.6$ & $87.7$ & $78.1$ & $81.9$ & $94.6$ & $77.2$ & $73.2$ \\
        Qwen3.7-Plus & $87.4$ & $85.7$ & $75.3$ & $79.2$ & $98.6$ & $78.0$ & $74.3$ \\
        Qwen3.6-Plus & $82.2$ & $78.7$ & $72.2$ & $75.2$ & $99.1$ & $75.6$ & $71.0$ \\
        Qwen3.6-27B & $79.9$ & $76.7$ & $71.9$ & $72.3$ & $96.3$ & $69.5$ & $67.5$ \\
        Qwen3.6-35A3B & $78.3$ & $74.4$ & $69.2$ & $71.3$ & $97.1$ & $71.3$ & $65.2$ \\
        \bottomrule
    \end{tabular}
\end{table}

\begin{figure*}[t]
    \centering
    \subfloat[Reasoning RL.\label{fig:rl_latency}]{%
        \includegraphics[width=0.36\linewidth]{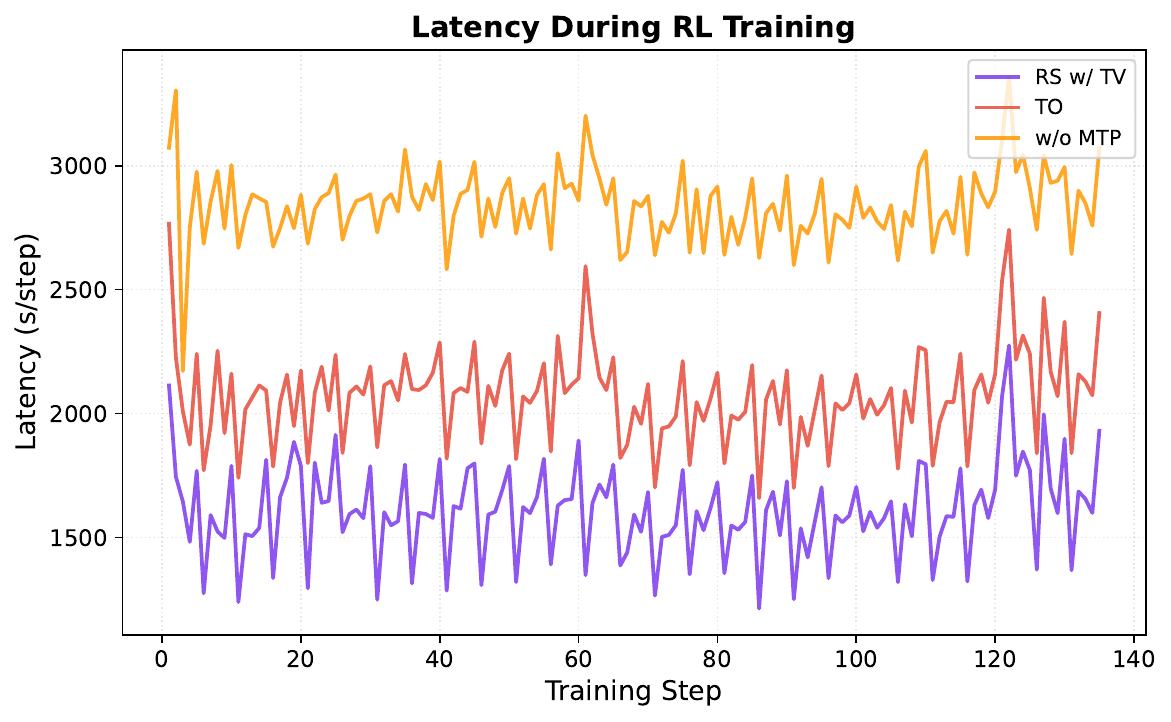}}
    \subfloat[SWE RL.\label{fig:rl_swe_latency}]{%
        \includegraphics[width=0.64\linewidth]{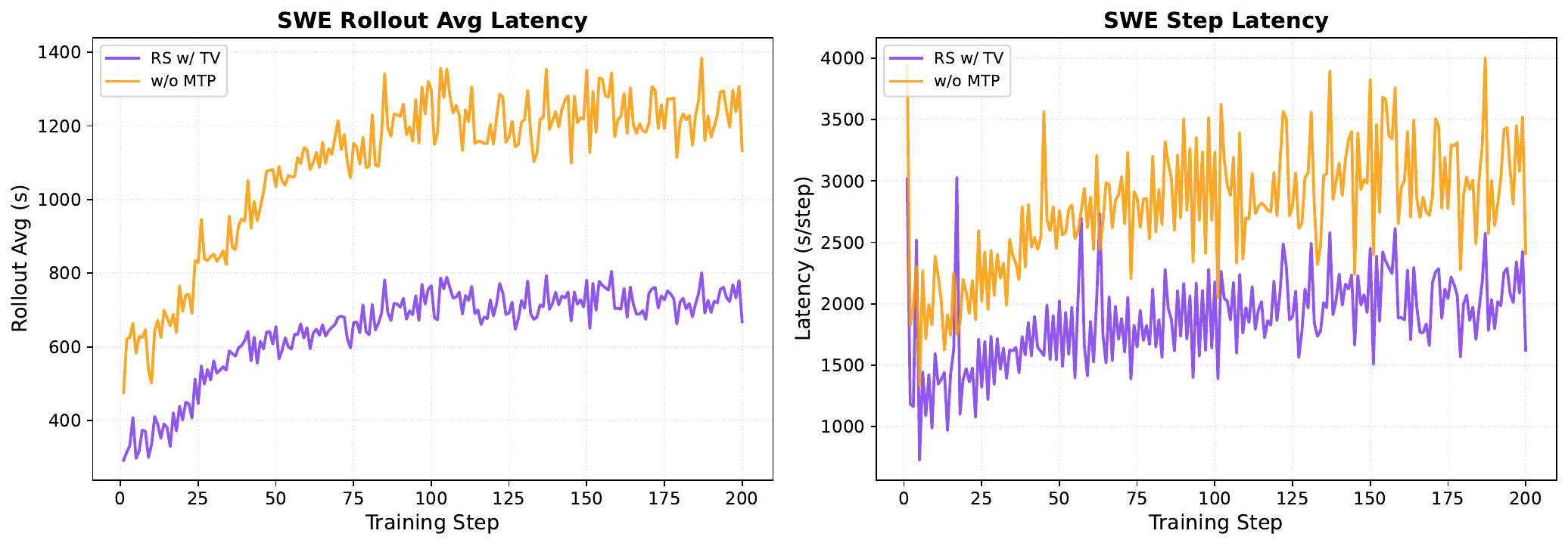}}\\
    \subfloat[Agent RL.\label{fig:rl_agent_latency}]{%
        \includegraphics[width=0.64\linewidth]{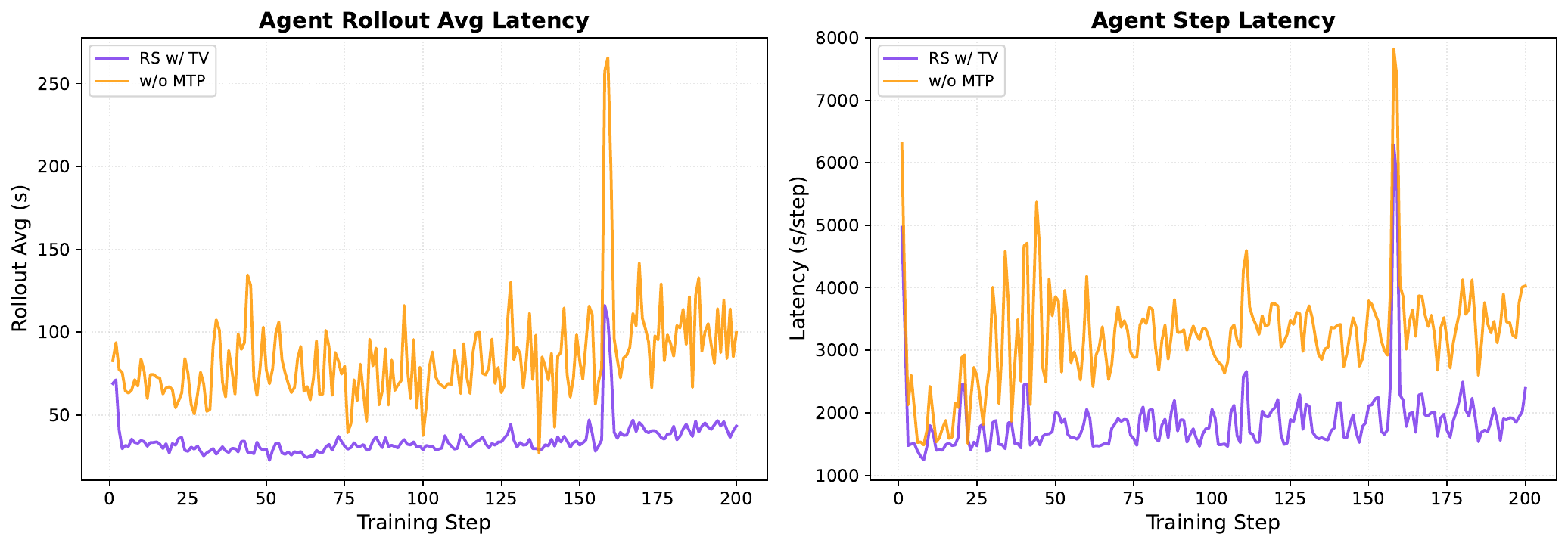}}
    \caption{Training latency comparison during RL using Qwen3.6-35A3B and Qwen3.6-Plus. MTP with rejection sampling (RS w/ TV) substantially reduces per-step latency compared to training without MTP (w/o MTP) and target-only sampling (TO).}
    \label{fig:rl_latency_combined}
\end{figure*}

\begin{figure}[t]
    \centering
    \subfloat[Reasoning RL]{
      \label{sfig:entropy_accept_hybrid}
      \includegraphics[width=0.32\linewidth]{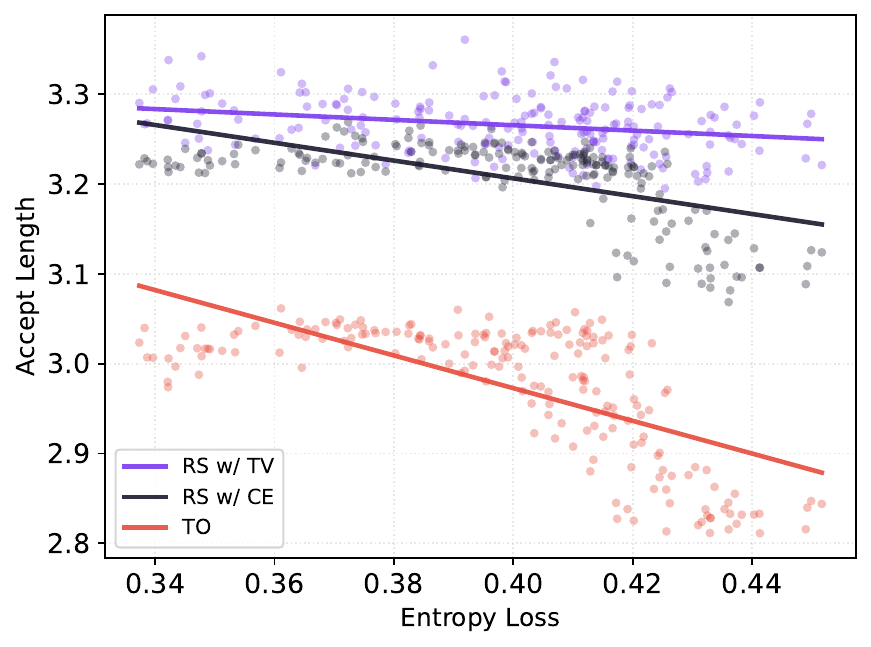}}
    \hfill
    \subfloat[SWE RL]{
      \label{sfig:entropy_accept_swe}
      \includegraphics[width=0.32\linewidth]{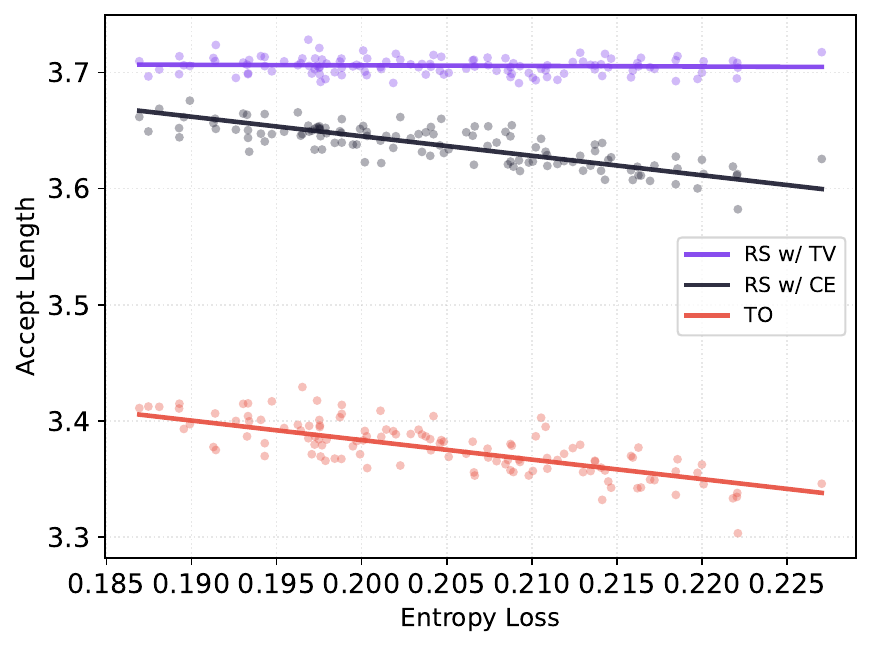}}
    \hfill
    \subfloat[SWE RL in Qwen-3.7 Max]{
      \label{sfig:entropy_accept_swe_max}
      \includegraphics[width=0.32\linewidth]{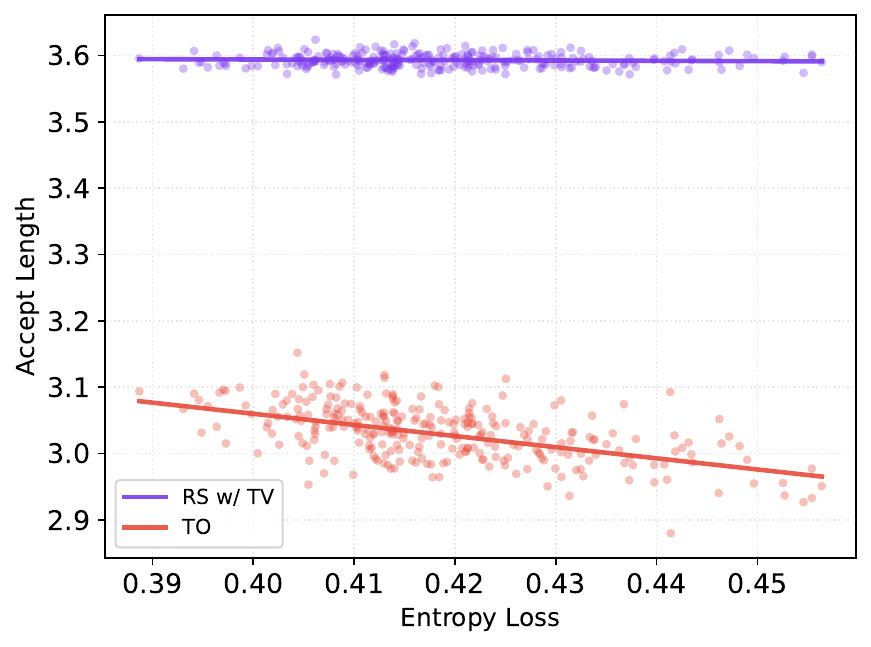}}
    \caption{Entropy loss vs.\ accept length across three RL workloads in Qwen3.6-Plus and Qwen3.7-Max.
        Each point represents one training step; the line shows the linear fit.
        TO and RS w/ CE exhibit a strong negative correlation (slope $\approx -1.68$), while RS w/ TV remains nearly flat (slope $\approx -0.06$), confirming that TV training decouples acceptance from entropy.}
    \label{fig:entropy_vs_accept}
\end{figure}

Fig.~\ref{fig:rl_accept_len} shows the accept length trends during RL training.
With rejection sampling and TV loss, \textit{Bebop} maintains stable or improving acceptance length throughout training, even as the policy maintains high entropy.
In Reasoning RL, the observed increase in acceptance rate is primarily driven by a significant drop in policy entropy during training, rather than improved draft alignment alone.
In contrast, the SWE workloads exhibit slightly increasing entropy, making them a more direct test of the training objective's robustness: here, RS w/ TV maintains stable accept lengths while target-only sampling suffers continuous degradation.
The advantage is most pronounced on SWE and other high-entropy tasks, where higher accept lengths translate directly into faster rollout completion.
Furthermore, at larger model scales (Fig.~\ref{fig:rl_swe_max_accept_len}), RS w/ TV exhibits a stronger entropy-invariant trend and sustains high acceptance rates throughout RL training, whereas target-only sampling shows a persistent acceptance rate decline.

Fig.~\ref{fig:rl_latency_combined} shows the corresponding latency improvements.
MTP with rejection sampling achieves $1.5$--$1.8\times$ reduction in per-step RL training latency compared to training without MTP, with the largest gains on agentic tasks where the rollout phase achieves up to $2.4\times$ speedup in Agentic RL.
These speedups are consistent across all workloads and provide substantial wall-clock savings at scale.

Fig.~\ref{sfig:entropy_vs_accept_length} and Fig.~\ref{fig:entropy_vs_accept} validate the linear entropy--acceptance relationships established in \S\ref{sec:formulation}.
Notably, training with TV loss substantially reduces the entropy--acceptance slope (by over $95\%$, e.g., from $-1.68$ to $-0.06$) and shifts the intercept upward.
This confirms that TV loss improves acceptance both by better aligning the draft distribution with the target and by largely decoupling the acceptance rate from the target entropy, consistent with the entropy-invariant mismatch structure analyzed in \S\ref{subsec:training_impact}, thereby enabling stable MTP acceleration gains throughout RL training.

\subsection{Benefits of Updating MTP Weights During RL}

After thorough multi-step SFT training, the model already achieves high acceptance rates (e.g., above 75\% for Qwen3.7-Max). As long as the acceptance rate is maintained, the MTP acceleration benefits are preserved throughout RL training. Furthermore, the analysis in \S\ref{sec:mtp_training} and the experimental validation in Fig.~\ref{sfig:entropy_vs_accept_length} demonstrate that rejection sampling with TV loss effectively decouples the entropy--acceptance relationship, stabilizing acceptance rates during RL.
To further quantify the benefits of updating MTP weights during RL, we compare the following training configurations:
\begin{enumerate}[label=(\arabic*), leftmargin=8mm, itemsep=1mm]
    \item RS w/ TV + TV loss: starting from the RS w/ TV checkpoint and online MTP training with TV loss;
    \item RS w/ TV + CE loss: starting from the RS w/ TV checkpoint and online MTP training with CE loss;
    \item TO + CE loss: starting from the TO checkpoint and continuing MTP training with CE loss.
\end{enumerate}

\begin{figure}[t]
    \centering
    \subfloat[Accept length with MTP weight updates.]{\includegraphics[width=0.48\columnwidth]{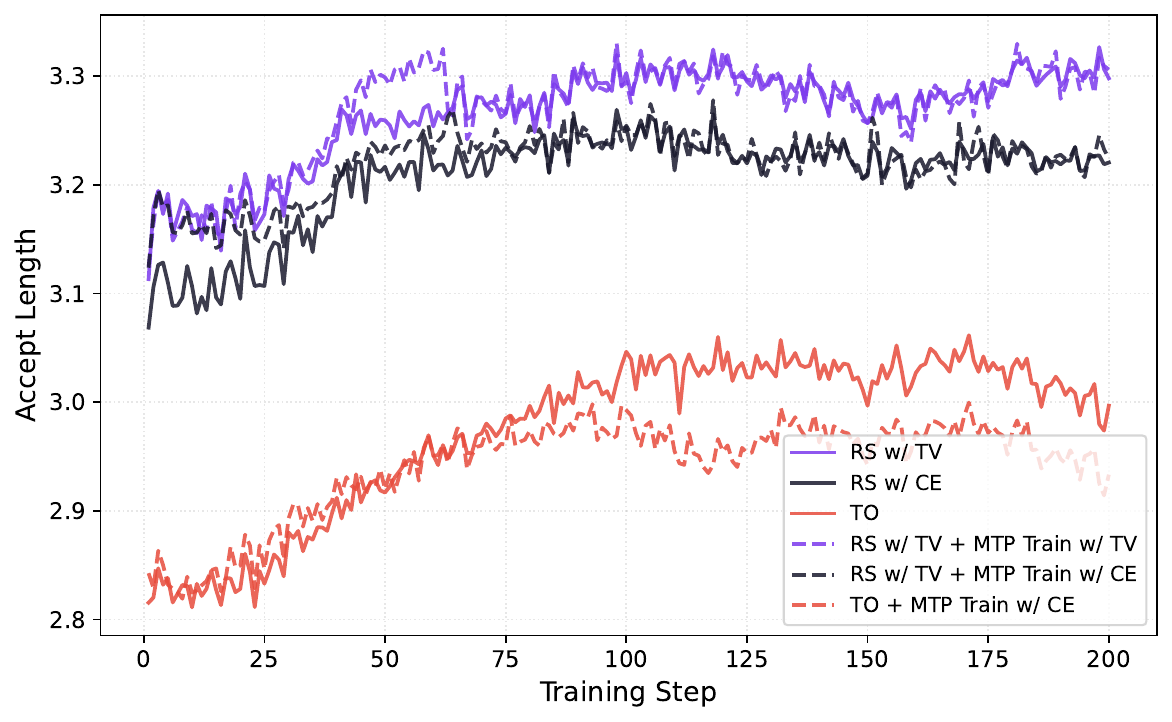}\label{fig:rl_mtp_train_accept_len}}
    \hfill
    \subfloat[Accept rate delta vs.\ throughput ratio.]{\includegraphics[width=0.48\columnwidth]{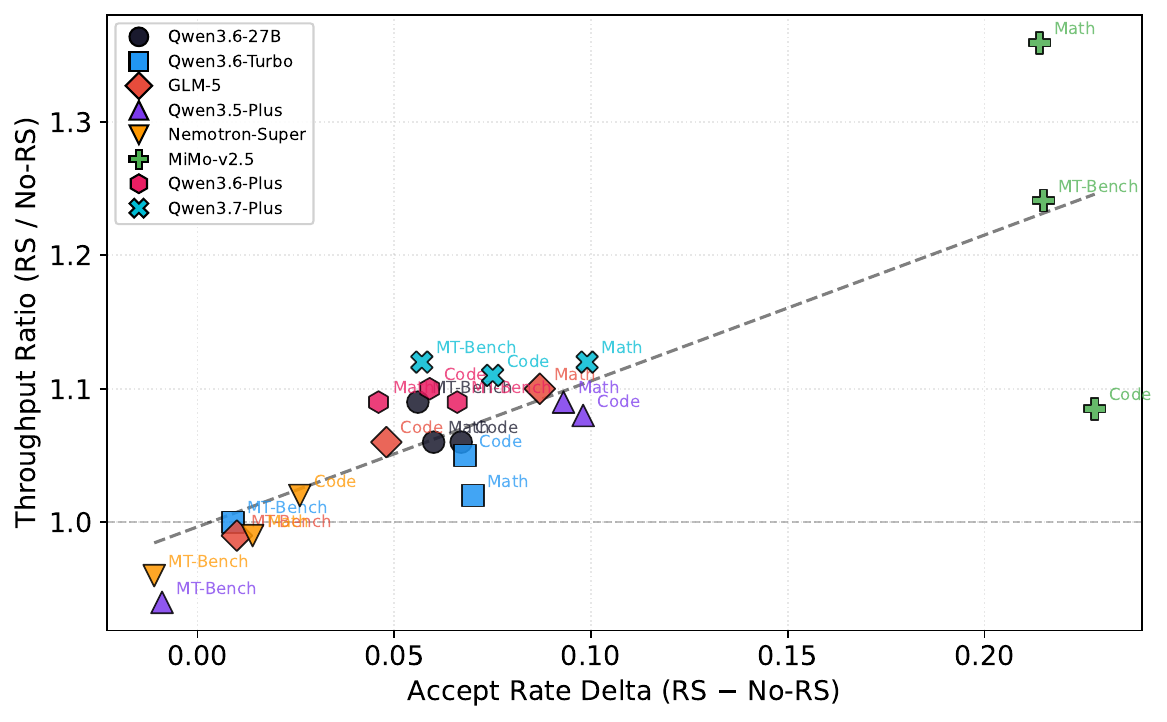}\label{sfig:accept_delta_vs_throughput}}
    \caption{(a) Accept length during RL training with and without MTP weight updates. Updating MTP weights with CE loss causes the acceptance rate to converge toward the corresponding non-updated baseline, while target-only sampling with CE loss updates can even degrade acceptance due to distribution mismatch. (b) Accept rate delta (RS $-$ No-RS) vs.\ throughput speedup ratio (RS / No-RS) across 8 models and 3 tasks ($r = 0.81$). Higher acceptance rate gains from rejection sampling translate directly to greater throughput improvements.}
    \label{fig:accept_position_and_mtp_train}
\end{figure}

As shown in Fig.~\ref{fig:rl_mtp_train_accept_len}, as RL training with MTP weight updates progresses, the acceptance rate converges toward the corresponding baseline without weight updates. For example, although RS w/ TV initially achieves a higher acceptance rate due to TV loss training, updating the MTP weights with CE loss during RL causes the acceptance rate to degrade toward that of RS w/ CE.
This shift in acceptance rate reflects changes in the draft distribution: as analyzed in \S\ref{subsec:distribution_shift_in_rl_training}, CE loss updates make the RS w/ TV draft distribution smoother, bringing it closer to the RS w/ CE distribution.
Moreover, for already well-trained MTP weights, further parameter updates during RL yield no significant improvement, with the acceptance rate closely tracking the non-updated baseline. In the case of target-only sampling, updating with CE loss can even cause acceptance rate degradation due to distribution mismatch between the draft and target models.
\section{Discussion}
\label{sec:discussion}

In this section, we provide a deeper analysis of the mechanisms behind e2e TV loss and rejection sampling, including the distributional effects of TV loss, comparison of the robustness of different acceptance methods, and analysis of how temperature, generation length, and agentic workloads affect MTP acceptance. 

\subsection{TV Loss Makes Draft Distributions Sharper}

We analyze how the TV loss affects the draft distribution's entropy compared to CE/KL training.
The TV loss produces draft distributions with entropy closer to the target entropy (but slightly higher), indicating that the draft becomes \textit{sharper} and more aligned with the target's peaked predictions.
In contrast, CE/KL training tends to produce smoother draft distributions that spread mass across the vocabulary, which is suboptimal for rejection sampling where the overlap $\sum_v \min(p(v), q(v))$ is maximized by matching the target's shape.

This sharpening effect arises from the TV loss gradient's selective behavior (Eq.~\eqref{equ:tv_grad}): it focuses optimization effort on tokens near the decision boundary ($q_j \approx p_j$) while ignoring irrelevant low-probability tokens.
Fig.~\ref{fig:entropy_gap_vs_kl} illustrates the relationship between the draft--target entropy gap and KL distance across models. Models with well-trained MTP heads exhibit a smaller entropy gap between draft and target distributions, while having a larger KL distance (see also Fig.~\ref{sfig:draft_target_distribution}).

\begin{figure*}[t]
    \centering
    \subfloat[Entropy gap vs.\ KL divergence.]{\includegraphics[width=0.33\textwidth]{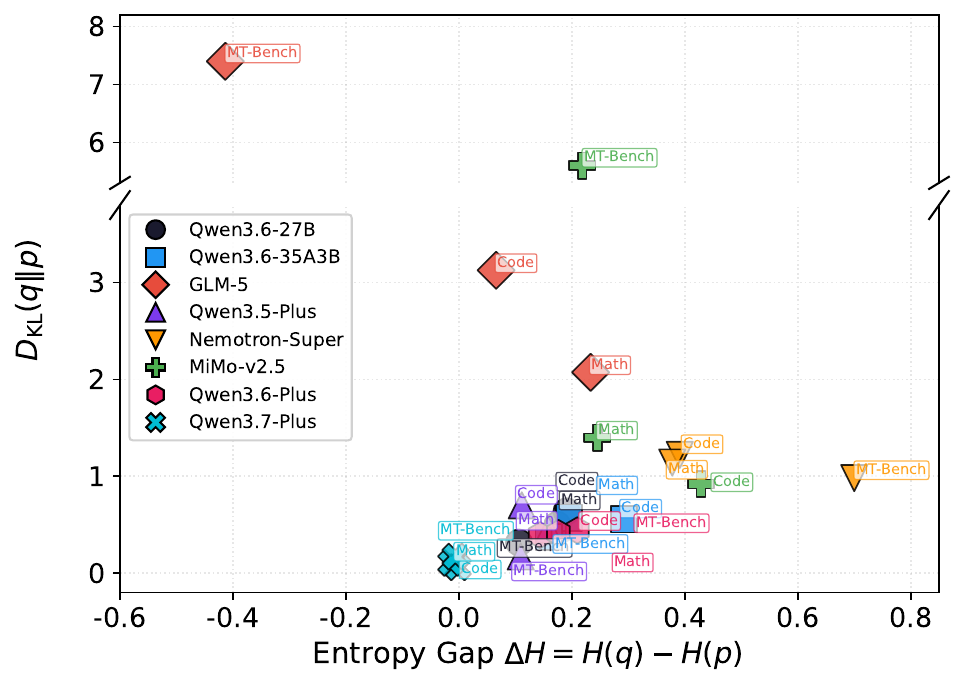}\label{sfig:entropy_gap_vs_kl}}
    \hfill
    \subfloat[Entropy gap vs.\ RS accept rate.]{\includegraphics[width=0.33\textwidth]{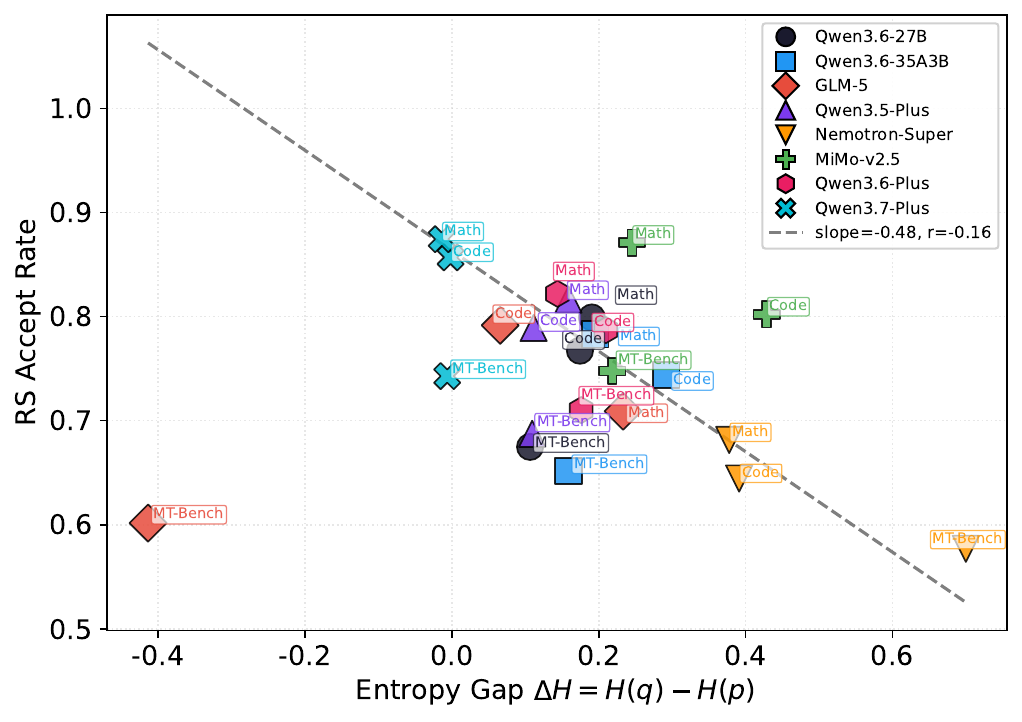}\label{sfig:entropy_gap_vs_rs_accept}}
    \hfill
    \subfloat[KL divergence vs.\ RS accept rate.]{\includegraphics[width=0.33\textwidth]{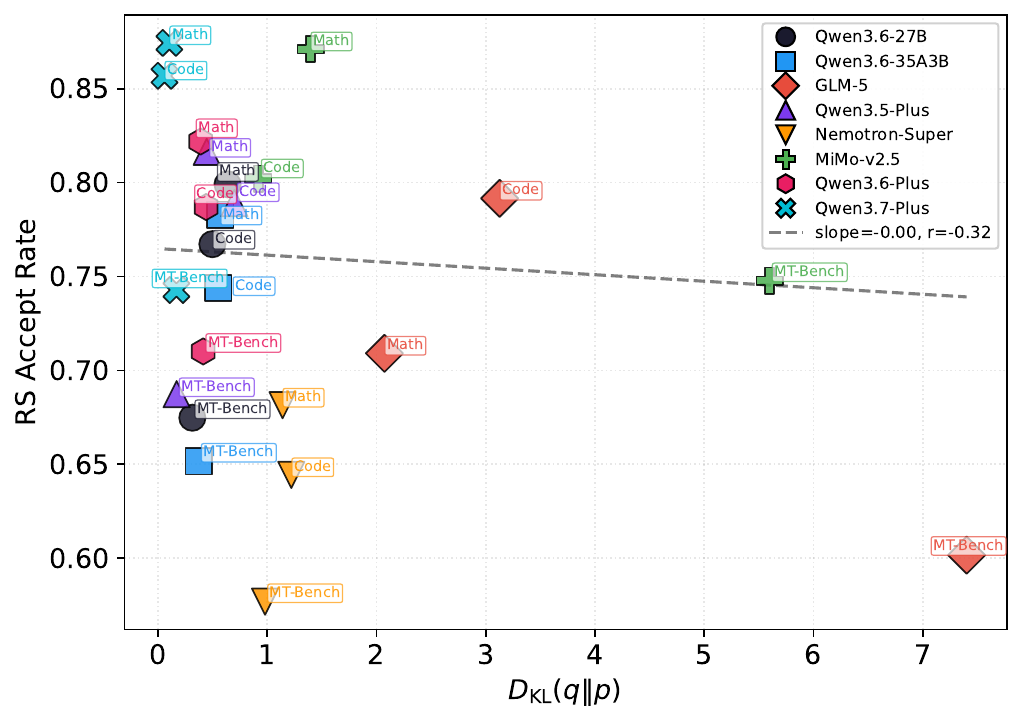}\label{sfig:kl_vs_rs_accept}}
    \caption{(a) Entropy gap $\Delta H$ vs.\ $D_{\mathrm{KL}}(q \| p)$ across models and tasks. (b) Entropy gap correlates negatively with RS acceptance rate ($r = -0.54$). (c) KL divergence shows no such correlation ($r = 0.13$), indicating that entropy gap, rather than KL, is the relevant predictor of RS acceptance.}
    \label{fig:entropy_gap_vs_kl}
\end{figure*}

\subsection{Different MTP Training Losses Induce Different Draft Distribution Patterns}
\label{subsec:distribution_shift_in_rl_training}

\begin{figure*}[t]
    \centering
    \includegraphics[width=\linewidth]{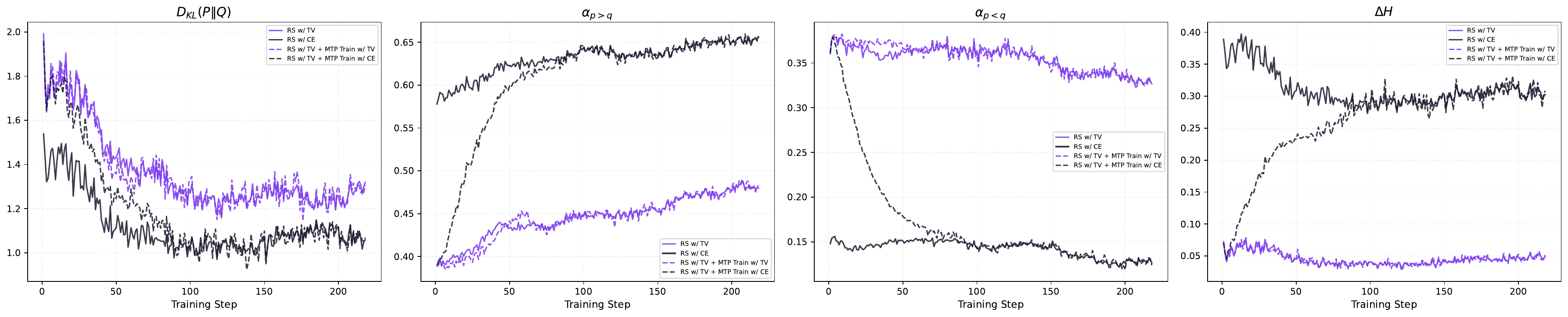}
    \caption{Evolution of MTP metrics during RL training with different MTP loss objectives. TV loss produces draft entropy closer to the target but with larger KL distance, lower $\alpha_{p > q}$, and higher $\alpha_{q > p}$. Switching the MTP training loss during RL causes the metrics to shift toward the pattern characteristic of the new loss.}
    \label{fig:rl_mtp_metrics}
\end{figure*}

Fig.~\ref{fig:rl_mtp_metrics} shows how various MTP metrics evolve when updating MTP weights with different losses during RL. TV loss produces draft entropy closer to the target model, but with a larger KL distance compared to CE loss.
Furthermore, because TV loss yields a sharper draft distribution, the corresponding $\alpha_{p > q}$ is lower while $\alpha_{q > p}$ is higher.
When different losses are used for MTP weight updates during RL, the MTP metrics shift toward the pattern characteristic of that loss. For example, with RS w/ TV + CE loss, the draft entropy gradually increases over the course of training.

\subsection{Robustness of Acceptance Methods under Policy Updates}
\label{subsec:robustness}

Although the analysis in \S\ref{subsec:decomposition} shows that the magnitude of model updates during RL is relatively small, an important distinction remains between target-only and rejection sampling in their sensitivity to ranking changes caused by RL policy updates.

\paragraph{Target-only sampling is fragile to ranking shifts.}
Target-only acceptance relies on whether the draft token falls within the target model's high-probability region (e.g., top-$k$).
This is a \textit{discrete} criterion: a token is either accepted or rejected.
When an RL gradient step causes the top-1 token to change, even by a small probability shift (e.g., $p(v_1)$ drops from $0.31$ to $0.29$ while $p(v_2)$ rises from $0.29$ to $0.31$), the draft model, still favoring the old top-1, experiences a discontinuous jump from acceptance to rejection.

\paragraph{Rejection sampling degrades smoothly.}
Under reject sampling, the acceptance rate $\alpha^{\mathrm{RS}} = \sum_v \min(p(v), q(v))$ is a \textit{continuous} function of both distributions.
The same ranking shift produces a negligible change in the TV overlap, since $\min(p(v_1), q(v_1)) + \min(p(v_2), q(v_2))$ is nearly invariant to small probability swaps.

\paragraph{High entropy amplifies the fragility gap.}
When the target entropy is high, multiple tokens have similar probabilities, making ranking changes more frequent under RL updates.
This disproportionately affects target-only sampling, where each ranking flip can cause a discrete acceptance failure.
Despite this qualitative difference, we empirically observe similar entropy--acceptance slopes for target-only and rejection sampling ($b^{\mathrm{TO}} \approx b^{\mathrm{RS}}$; see \S\ref{sec:formulation}), suggesting that the discrete fragility of target-only is offset by the cumulative TV distance growth that affects rejection sampling equally under CE/KL training.

\subsection{Correlation between Temperature and MTP Acceptance}
\label{subsec:temperature}

The sampling temperature $\tau$ directly affects the target model's entropy: $\mathcal{H}(p_\tau) = \mathcal{H}(\mathrm{softmax}(z / \tau))$ increases monotonically with $\tau$.
Combined with the linear entropy-acceptance relationship established in \S\ref{sec:formulation}, this implies that \textit{higher temperatures lead to lower MTP acceptance rates}.

\begin{figure}[t]
    \centering
    \subfloat[Acceptance length vs.\ temperature.]{\includegraphics[width=0.48\columnwidth]{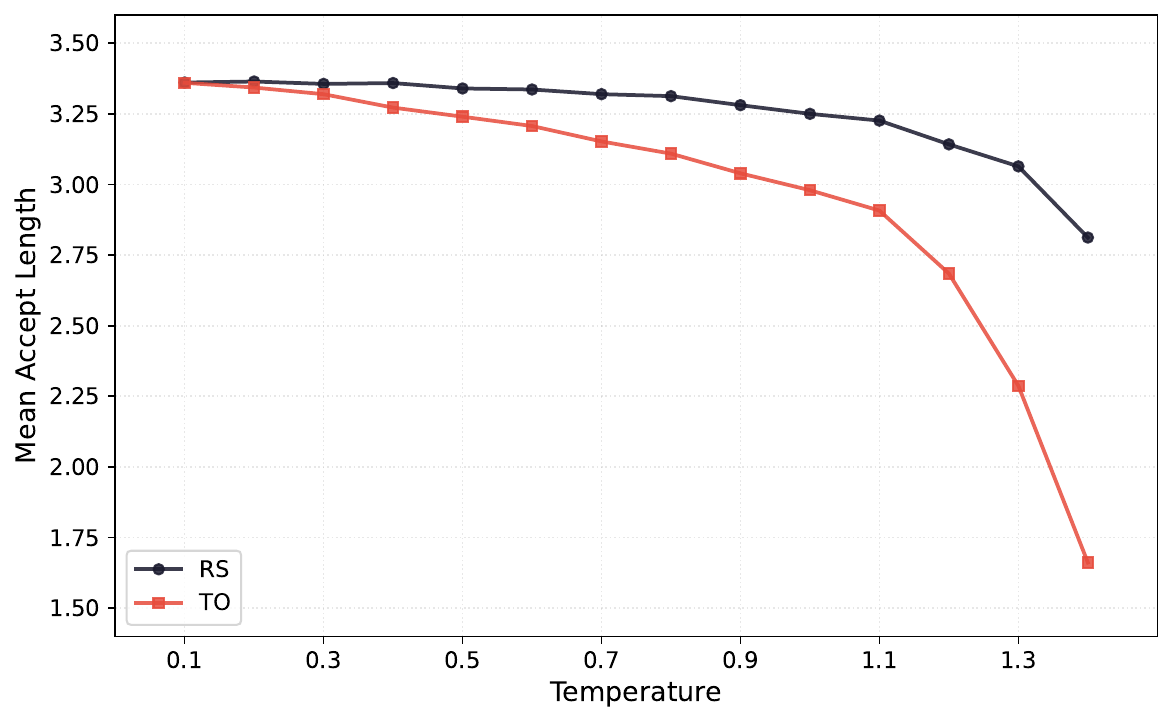}\label{sfig:accept_vs_temp}}
    \hfill
    \subfloat[Accept rate vs.\ output length.]{\includegraphics[width=0.48\columnwidth]{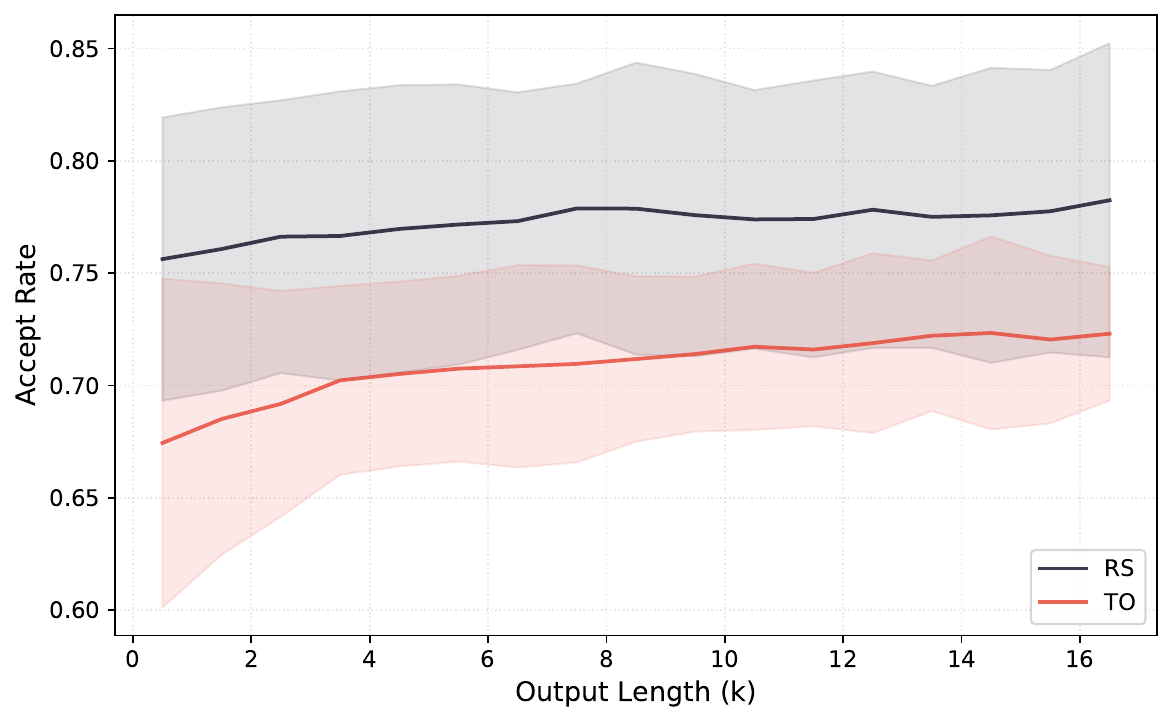}\label{sfig:accept_vs_position}}
    \caption{(a) Mean acceptance length as a function of sampling temperature. Rejection sampling maintains relatively stable acceptance lengths, while target-only sampling degrades sharply at higher temperatures. (b) MTP acceptance rate vs.\ output length (averaged over 8 models). RS maintains a stable advantage over target-only sampling across all generation positions.}
    \label{fig:temp_and_position}
\end{figure}

\begin{figure}[t]
    \centering
    \includegraphics[width=0.6\columnwidth]{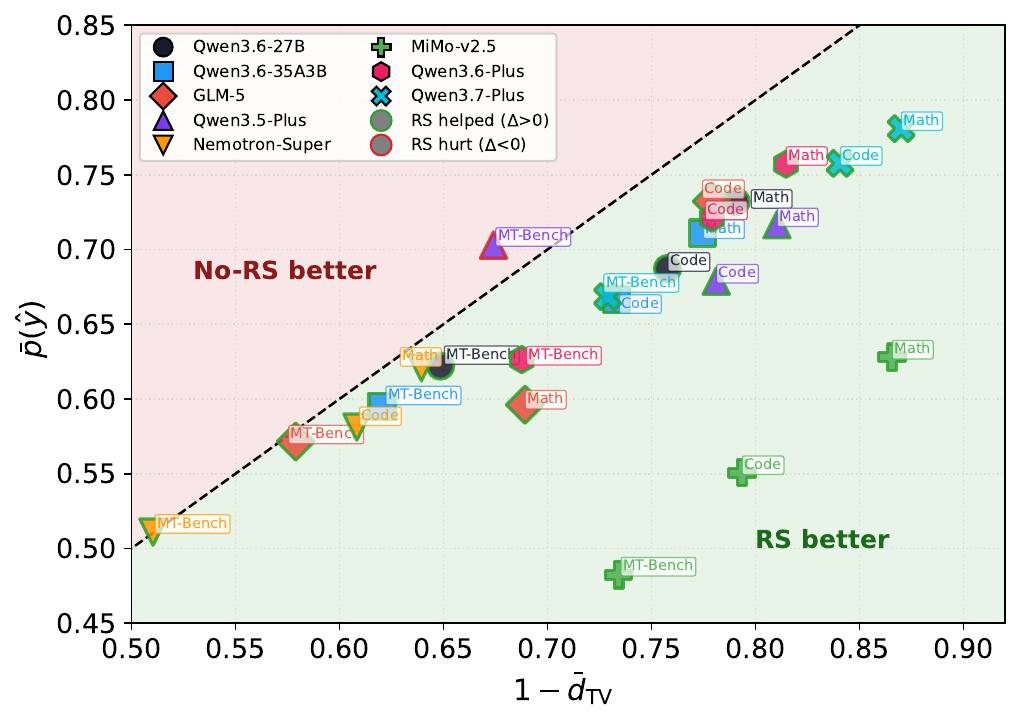}
    \caption{RS decision boundary across models (see \S\ref{subsec:rs_boundary}). Nearly all model--task combinations fall in the RS-better region, confirming that rejection sampling is beneficial for virtually all practical MTP deployments.}
    \label{sfig:rs_boundary}
\end{figure}

Fig.~\ref{sfig:accept_vs_temp} confirms this: rejection sampling maintains relatively stable acceptance lengths across temperatures, while target-only sampling degrades sharply as temperature increases.
This has practical implications for RL training, where higher temperatures are often used to encourage exploration.
Our analysis provides a quantitative framework for understanding the throughput cost of exploration via temperature scaling.

\subsection{Rejection Sampling Decision Boundary}
\label{subsec:rs_boundary}

Rejection sampling outperforms target-only sampling when
\(d_{\mathrm{TV}}(p, q) < 1 - p(\hat{y})\),
with $\hat{y} \!=\! \arg\max_y q(y)$ (see \S\ref{app:rs_boundary}).
This decision boundary provides a simple diagnostic: if the draft--target TVD is smaller than the probability mass outside the draft's top-1 token under the target, RS is preferred.

Fig.~\ref{sfig:rs_boundary} visualizes this boundary across eight models with natively trained MTP heads, spanning three task categories.
Nearly all model--task combinations (23 out of 24) fall firmly in the RS-better region, confirming that for native MTP models, rejection sampling consistently outperforms target-only sampling.
This confirms that enabling rejection sampling is beneficial for virtually all practical MTP deployments.

\subsection{Correlation between Generation Length and MTP Acceptance}
\label{subsec:gen_length}

As shown in Fig.~\ref{sfig:accept_vs_position}, we observe that MTP acceptance rates vary systematically with the position in the generated sequence.
In early positions (close to the prompt), the target model tends to have lower entropy (more predictable continuations), leading to higher acceptance rates.
As generation progresses, especially in reasoning tasks with long chains of thought, entropy can increase and acceptance rates may drop.
This position-dependent acceptance pattern suggests that \textit{adaptive MTP strategies}---adjusting the draft length $\gamma$ based on the estimated local entropy---could further improve throughput.

\subsection{Agentic RL and the Bubble Problem}
\label{subsec:agentic}

As shown in Fig.~\ref{fig:agent_accept_len}, in agentic RL settings (e.g., SWE-bench~\citep{jimenez2024swebench}), the model generates long, multi-turn interactions that involve tool calls, code execution, and iterative refinement.
These settings exhibit particularly long generation lengths and variable entropy profiles, creating periodic fluctuations in acceptance rate that tend to increase as generation progresses.

\begin{figure}[t]
    \centering
    \subfloat[Accept length during Agent RL.]{\includegraphics[width=0.48\columnwidth]{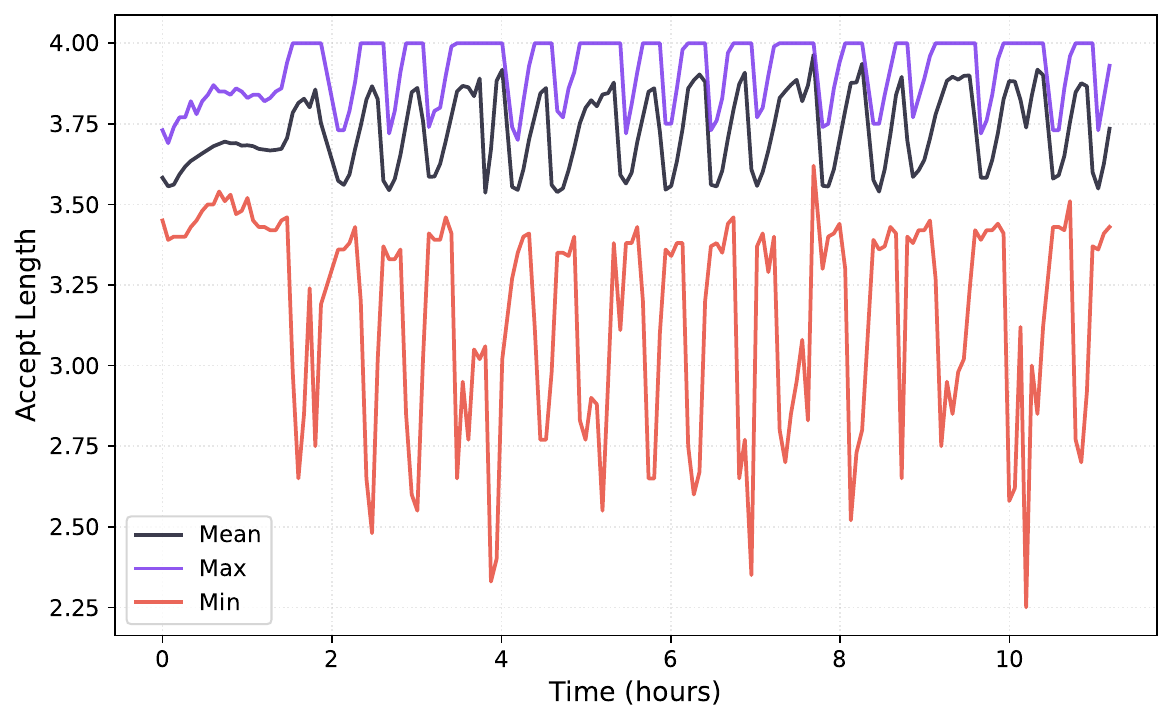}\label{fig:agent_accept_len}}
    \hfill
    \subfloat[MTP loss under top-$K$ truncation.]{\includegraphics[width=0.48\columnwidth]{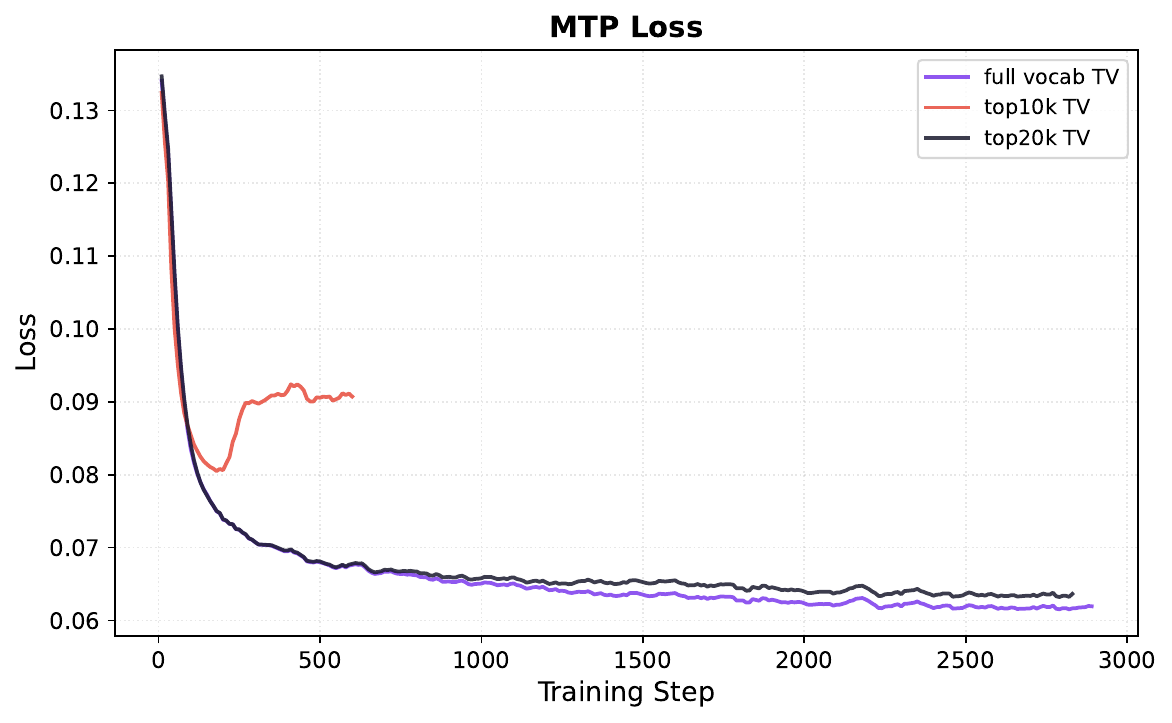}\label{fig:mtp_lm_loss_topk}}
    \caption{(a) Accept length during Agent RL. The mean acceptance length remains stable at $\sim$3.7, while the min--max range reveals periodic fluctuations across steps. (b) MTP loss curves under different top-$K$ truncation values. Smaller $K$ leads to pronounced loss spikes and training instability, while even $K = 20{,}000$ shows slower convergence compared to the full-vocabulary TV loss.}
    \label{fig:agent_and_topk}
\end{figure}

MTP is especially beneficial in agentic settings for two reasons:
(1) long generations contain abundant structured outputs---such as boilerplate code, tool call formats, and repetitive patterns---that are highly predictable, yielding high acceptance rates in these segments;
(2) multi-turn interactions and long-tail generation reduce the effective running batch size, a regime where MTP's latency benefits are amplified since the inference engine operates further from compute saturation.
Indeed, our experiments show that agentic workloads achieve the largest acceptance rate improvements (5\%) from our proposed TV loss training.

\subsection{Instability of Top-K TV Approximation}

Computing the full-vocabulary TV loss incurs high peak memory on large vocabularies. To address this, we employ a fused backward kernel that reduces intermediate activation sizes (see \S\ref{app:tv_kernel}). We also experimented with approximating TV loss via a top-$K$ truncation to further reduce peak memory. However, even with $K = 20{,}000$, we observe a slight slowdown in loss convergence and performance degradation. Smaller values of $K$ lead to pronounced loss spikes, as shown in Fig.~\ref{fig:mtp_lm_loss_topk}.
Ultimately, we adopt the fused full-vocabulary TV loss rather than the top-$K$ approximation.
\section{Related Work}
\label{sec:related}

\paragraph{Speculative Decoding.}
Speculative decoding \citep{leviathan2023fast,chen2023accelerating} accelerates autoregressive LLM inference by using a lightweight draft model to propose multiple tokens, which are then verified by the target model in parallel.
Various draft architectures have been proposed, including independent small models \citep{specinfer-miao-2024,shen2026specbranch}, early-exit heads \citep{elhoushi2024layerskip}, auxiliary heads \citep{medusa-cai-2024,li2025eagle,li2025eagle3}, MTP heads \citep{dpsk-v3,gloeckle2024bmtp,qwen3.5}, and diffusion models~\citep{chen2026dflash}.
\textit{Bebop} focuses on MTP heads that share the backbone's hidden states and analyzes their behavior under RL training dynamics.

\paragraph{Reinforcement Learning for LLMs.}
RL has become central to aligning LLMs with human preferences~\citep{ppo} and enhancing reasoning and agentic capabilities~\citep{gpt55,deepseek-v4,qwen37}.
Modern post-training pipelines typically separate RL into rollout, reward evaluation, and policy update stages, while algorithms such as GRPO~\citep{grpo} and GSPO~\citep{gspo} improve the optimization objective itself.
At the system level, asynchronous or partial-rollout frameworks reduce idle time from long-tail trajectories by decoupling inference workers from training workers~\citep{fu2025areal,wang2025roll,slime2025,qin2025seer}.
Yet they mainly hide long-tail bubbles, leaving trajectory generation as the bottleneck in long-context, multi-turn, and tool-use settings.
Related work studies RL instability from training-inference discrepancy and policy staleness~\citep{yao2025offpolicy,liu-li-2025-rl-collapse}; recent MTP methods instead update draft heads online to address draft--target mismatch~\citep{chen2025respec,iso2026accelerating,minimax2026forge,li2025mtp}.
However, we find that acceptance rate fluctuations during RL are primarily driven by shifts in the target model's entropy rather than draft--target mismatch, and that target entropy exhibits a linear relationship with MTP acceptance length, an observation also noted by~\citet{mimo-v2-flash}.
Our work is complementary: it accelerates rollout without changing the RL objective or scheduler, and identifies entropy shifts as the dominant factor behind MTP acceptance degradation.

\paragraph{Total Variation Distance in Machine Learning.}
The TV distance is a standard measure for comparing probability distributions, and has been used in distribution testing~\citep{canonne2020survey}, generative modeling~\citep{nowozin2016fgan}, and convergence analysis of Markov chains~\citep{levin2017markov}.
In speculative decoding, the rejection-sampling acceptance rate equals the distributional overlap, i.e., $\alpha=\sum_y \min(p_y,q_y)=1-d_{\mathrm{TV}}(p,q)$~\citep{leviathan2023fast,chen2023accelerating}.
This connection has motivated acceptance-oriented objectives, including LK Losses for directly optimizing speculative decoding acceptance rate~\citep{samarin2026lk}.
However, these works focus on inference-time speculative decoding with a fixed target model.
Recent work has also explored using reverse KL to optimize the student model in OPD~\citep{lu2025onpolicydistillation,lei2026draftopd}, though its training objective still differs substantially from directly maximizing the rejection sampling acceptance rate.
To our knowledge, we are the first to propose directly optimizing TV distance as a training objective for MTP heads, and the first to analyze its behavior during RL training.

\section{Conclusion}

We present \textit{Bebop}, a systematic study of Multi-Token Prediction (MTP) in the context of reinforcement learning for large language models.
Our analysis reveals three key findings:
(1) MTP acceptance rates under both target-only and rejection sampling are linearly constrained by the target model's entropy;
(2) \textit{Bebop}'s end-to-end TV loss directly optimizes multi-step rejection sampling acceptance, yielding ${\sim}10\%$ acceptance-rate improvements, up to 95\% acceptance, and up to 25\% extra inference throughput over conventional CE/KL objectives;
(3) Lightweight pre-RL adaptation with TV loss and rejection sampling is sufficient to maintain high MTP acceptance rates throughout RL training, eliminating the need for costly online MTP updates.
Extensive experiments with Qwen3.5, 3.6, and 3.7 models demonstrate that \textit{Bebop} achieves up to $1.8\times$ end-to-end acceleration in async RL pipelines.

\paragraph{Limitations.}
Our theoretical analysis of the entropy-acceptance relationship relies on modeling assumptions (uniform vs.\ probability-proportional mismatch) that are heuristically motivated by gradient structures rather than formally proven; tightening these assumptions remains an open question.
Additionally, the entropy invariance guaranteed by TV training is distribution-conditional: it holds within the entropy range covered by the SFT training data, but when RL exploration drives the policy entropy significantly beyond this range, the draft head encounters out-of-distribution target distributions for which the mismatch ratio $\delta$ is no longer bounded, restoring an entropy-acceptance dependence comparable to that of CE/KL training.
In such cases, MTP co-training with TV loss during RL is recommended to extend the draft head's effective coverage to the new entropy regime.

\bibliographystyle{plainnat}
\bibliography{reference}

\appendix

\section{Derivation of TV Loss Gradient}
\label{app:tv_grad}

We provide the full derivation of the TV loss gradient (Eq.~\eqref{equ:tv_grad}).

Let the draft head output logits $z \in \mathbb{R}^{|\mathcal{V}|}$ with $q_j = \mathrm{softmax}(z)_j = \frac{e^{z_j}}{\sum_k e^{z_k}}$.
The target model probability $p$ is treated as a constant (detached).
The TV loss is:
\begin{align*}
\mathcal{L}_{\mathrm{TV}} = 1 - \sum_{v} \min(p_v, q_v).
\end{align*}

The gradient with respect to $z_j$ is:
\begin{align*}
\frac{\partial \mathcal{L}_{\mathrm{TV}}}{\partial z_j} = -\frac{\partial}{\partial z_j} \sum_{v} \min(p_v, q_v).
\end{align*}

Since $p$ is constant, the subgradient of $\min(p_v, q_v)$ with respect to $q_v$ is:
\begin{align*}
\frac{\partial}{\partial q_v} \min(p_v, q_v) = \mathbbm{1}[q_v \leq p_v].
\end{align*}

Using the chain rule with the softmax Jacobian $\frac{\partial q_v}{\partial z_j} = q_v(\delta_{vj} - q_j)$:
\begin{align*}
\frac{\partial}{\partial z_j} \sum_{v} \min(p_v, q_v) &= \sum_{v} \mathbbm{1}[q_v \leq p_v] \cdot q_v (\delta_{vj} - q_j) \\
&= \underbrace{\mathbbm{1}[q_j \leq p_j] \cdot q_j}_{v = j \text{ term}} - q_j \underbrace{\sum_{v} \mathbbm{1}[q_v \leq p_v] \cdot q_v}_{\triangleq\; S} \\
&= q_j \bigl[ \mathbbm{1}[q_j \leq p_j] - S \bigr].
\end{align*}

Therefore:
\begin{align*}
\frac{\partial \mathcal{L}_{\mathrm{TV}}}{\partial z_j} = -q_j \bigl[ \mathbbm{1}[q_j \leq p_j] - S \bigr],
\end{align*}
where $S = \sum_{v} \mathbbm{1}[q_v \leq p_v] \cdot q_v \in [0, 1]$.

\paragraph{Boundedness.}
Since $q_j \in [0, 1]$ and $|\mathbbm{1}[q_j \leq p_j] - S| \leq 1$:
\begin{align*}
\left| \frac{\partial \mathcal{L}_{\mathrm{TV}}}{\partial z_j} \right| = q_j \cdot |\mathbbm{1}[q_j \leq p_j] - S| \leq q_j \leq 1.
\end{align*}

\section{Comparison with Forward KL Divergence Gradient}
\label{app:kl_comparison}

For comparison, the gradient of the forward KL divergence $D_{\mathrm{KL}}(p \| q) = \sum_v p_v \log \frac{p_v}{q_v}$ with respect to $z_j$ is:
\begin{align*}
\frac{\partial D_{\mathrm{KL}}(p \| q)}{\partial z_j} = q_j - p_j.
\end{align*}

Key differences from the TV loss gradient:
\begin{enumerate}
    \item The forward KL gradient applies a nonzero force to \textit{every} token where $q_j \neq p_j$, including tokens with negligible probability. The TV gradient is proportional to $q_j$, so it automatically ignores low-probability tokens.
    \item The forward KL gradient does not distinguish between tokens that would be accepted vs.\ rejected under rejection sampling. The TV gradient explicitly incorporates this distinction via the indicator $\mathbbm{1}[q_j \leq p_j]$.
    \item The forward KL gradient can be large when $q_j \gg p_j$ (overconfident draft). The TV gradient is bounded by $q_j$.
\end{enumerate}

\section{Analysis of the Reverse KL Divergence}
\label{app:reverse_kl}

The preceding analysis focuses on the forward KL divergence $D_{\mathrm{KL}}(p\|q)$, which is equivalent to CE loss up to a constant.
A natural question is whether the reverse KL divergence $D_{\mathrm{KL}}(q\|p) = \sum_v q_v \log \frac{q_v}{p_v}$ would be a better training objective for rejection sampling.

\paragraph{Gradient derivation.}
The gradient of the reverse KL divergence with respect to the draft logits $z_j$ is:
\begin{align}
\frac{\partial D_{\mathrm{KL}}(q \| p)}{\partial z_j}
&= \sum_v \bigl[\log(q_v / p_v) + 1\bigr] \cdot q_v(\delta_{vj} - q_j) \nonumber \\
&= q_j \bigl[\log(q_j / p_j) + 1\bigr] - q_j \sum_v q_v \bigl[\log(q_v / p_v) + 1\bigr] \nonumber \\
&= q_j \bigl[\log(q_j / p_j) - D_{\mathrm{KL}}(q \| p)\bigr].
\label{eq:rkl_gradient}
\end{align}

\paragraph{Comparison of gradient structures.}
Table~\ref{tab:grad_comparison} summarizes the three gradient structures.

The reverse KL gradient shares the desirable $q_j$-proportionality with the TV gradient, meaning low-probability tokens automatically receive negligible optimization pressure.
This suggests that the reverse KL should produce a mismatch that scales more proportionally with $q_j$ than the uniform mismatch of forward KL, and consequently exhibit weaker entropy--acceptance coupling than the forward KL.

\paragraph{Why reverse KL is still suboptimal.}
Despite the improved gradient structure, the reverse KL remains suboptimal for maximizing the rejection sampling acceptance rate for three reasons:

\begin{enumerate}[leftmargin=8mm, itemsep=2mm]
    \item \textbf{Zero-forcing behavior.}
    The reverse KL does not penalize $q(v) \to 0$ even when $p(v) > 0$, since $\lim_{q \to 0} q \log(q/p) = 0$.
    This ``mode-seeking'' property allows the draft to drop modes of $p$, directly forfeiting the overlap $\min(p(v), q(v))$ at those tokens and reducing the acceptance rate.
    In contrast, the forward KL is ``zero-avoiding'' ($D_{\mathrm{KL}}(p\|q) \to \infty$ when $q(v) \to 0$ with $p(v) > 0$), enforcing full support coverage.
    The TV loss is neither zero-forcing nor zero-avoiding: it selectively allocates capacity to tokens where the marginal overlap improvement is largest.

    \item \textbf{Asymmetric over-/under-estimation penalty.}
    The acceptance ratio of rejection sampling depends on $\sum_v \min(p(v), q(v))$, which penalizes over-estimation ($q > p$) and under-estimation ($q < p$) symmetrically---both reduce the overlap by $|q(v) - p(v)|$.
    The reverse KL imposes an \textit{asymmetric} penalty: over-estimation ($q_j > p_j$, so $\log(q_j/p_j) > 0$) incurs a much stronger gradient than under-estimation.
    This drives the draft toward $q(v) \leq p(v)$ across most tokens, which ensures individual-token acceptance probability $\min(1, p/q) = 1$ but reduces the sampling probability of those tokens, yielding suboptimal total overlap.

    \item \textbf{Indirect optimization target.}
    Like the forward KL, the reverse KL does not directly optimize $d_{\mathrm{TV}}(p, q)$.
    The log-ratio $\log(q_j/p_j)$ in the reverse KL gradient provides a soft, nonlinear signal, whereas the TV gradient's indicator $\mathbbm{1}[q_j \leq p_j]$ provides a hard, direct signal aligned with the rejection sampling decision boundary.
\end{enumerate}

\paragraph{Summary.}
In terms of suitability for optimizing rejection sampling acceptance rates:
\begin{align*}
\text{TV loss} \;>\; \text{Reverse KL} \;>\; \text{Forward KL (CE)}.
\end{align*}
The reverse KL improves upon the forward KL through better capacity allocation (gradient $\propto q_j$), but remains suboptimal due to its zero-forcing behavior and asymmetric penalty structure.
The TV loss directly optimizes the quantity of interest and avoids both failure modes.

\section{Entropy-Acceptance Relationship under Different Training Objectives}
\label{app:entropy_accept}

We provide a detailed analysis of how the target model's entropy $\mathcal{H}(p)$ constrains MTP acceptance rates under different acceptance methods and training objectives.

\subsection{Setup and Notation}

Consider a fixed position $t$ in the generation process.
Let $p \in \Delta^{|\mathcal{V}|}$ denote the target model's distribution and $q \in \Delta^{|\mathcal{V}|}$ the draft model's distribution.
The draft model is parameterized by $q_\theta$ with logits $z \in \mathbb{R}^{|\mathcal{V}|}$ and $q_j = \mathrm{softmax}(z)_j$.
Due to finite model capacity, the draft cannot perfectly match $p$ in general, and the per-token mismatch structure depends critically on the training objective.

We define the \textit{effective support} of $p$ at threshold $\tau$ as $\mathcal{S}_\tau(p) = \{v \in \mathcal{V} : p(v) > \tau\}$, and recall that the effective support size is related to entropy via the perplexity: $|\mathcal{S}_{\mathrm{eff}}(p)| \approx \exp(\mathcal{H}(p))$.

For the analysis below, we consider two mismatch structures depending on the training objective:
\begin{itemize}[leftmargin=8mm, itemsep=1mm]
    \item \textbf{Uniform mismatch} (CE/KL training): $q^*(v) = p(v) + \eta_v$ with $|\eta_v| \lesssim \sigma$ and $\sum_v \eta_v = 0$, where $\sigma$ is approximately uniform across tokens (see \S\ref{subsec:ce_rs} for justification).
    \item \textbf{Probability-proportional mismatch} (TV training): $|q^*(v) - p(v)| \lesssim \delta \cdot p(v)$, where the absolute error scales with the token probability (derived under the capacity-allocation assumption in \S\ref{subsec:tv_rs}).
\end{itemize}

\subsection{Target-Only Sampling}
\label{subsec:to_analysis}

Under target-only sampling, the draft token is selected greedily as $\hat{y} = \arg\max_y q(y)$ and accepted with probability $p(\hat{y})$, giving acceptance rate:
\begin{align}
\alpha^{\mathrm{TO}} = p\!\left(\arg\max_y\, q(y)\right).
\end{align}

\paragraph{Perfect draft case.}
For a well-trained draft model where $\arg\max_y q(y) = \arg\max_y p(y)$ (i.e., the draft correctly identifies the target's top-1 token), the acceptance rate reduces to:
\begin{align}
\alpha^{\mathrm{TO}} = \max_y\, p(y).
\end{align}

\paragraph{Relationship to Shannon entropy.}
The quantity $\max_y p(y)$ is a monotonically decreasing function of $\mathcal{H}(p)$: as entropy increases, the distribution spreads and the maximum probability decreases.
A standard bound gives $\max_y p(y) \geq \exp(-\mathcal{H}(p))$, so the acceptance rate is lower-bounded by $\exp(-\mathcal{H}(p))$.

\paragraph{Linearization.}
Since $\alpha^{\mathrm{TO}} = \max_y p(y)$ is a smooth, monotonically decreasing function of $\mathcal{H}(p)$, we can write $\alpha^{\mathrm{TO}} = f(\mathcal{H}(p))$ for some decreasing function $f$.
Performing a first-order Taylor expansion around the mean operating entropy $\bar{\mathcal{H}} = \frac{1}{2}(\mathcal{H}_{\min} + \mathcal{H}_{\max})$:
\begin{align}
\alpha^{\mathrm{TO}} = f(\mathcal{H}) &\approx f(\bar{\mathcal{H}}) + f'(\bar{\mathcal{H}}) \cdot (\mathcal{H}(p) - \bar{\mathcal{H}}) \nonumber \\
&= \underbrace{\bigl[f(\bar{\mathcal{H}}) - f'(\bar{\mathcal{H}})\bar{\mathcal{H}}\bigr]}_{a^{\mathrm{TO}}} + \underbrace{f'(\bar{\mathcal{H}})}_{-b^{\mathrm{TO}}} \cdot \mathcal{H}(p).
\end{align}
Since $f$ is decreasing, $f'(\bar{\mathcal{H}}) < 0$, so $b^{\mathrm{TO}} = -f'(\bar{\mathcal{H}}) > 0$, yielding:
\begin{align}
\alpha^{\mathrm{TO}} \approx a^{\mathrm{TO}} - b^{\mathrm{TO}} \cdot \mathcal{H}(p).
\end{align}
The lower bound $f(\mathcal{H}) \geq \exp(-\mathcal{H})$ provides an order-of-magnitude estimate for the slope: $b^{\mathrm{TO}} \sim \exp(-\bar{\mathcal{H}})$.
Empirically, this linear approximation is remarkably robust across model scales, tasks, and training stages (Fig.~\ref{sfig:entropy_vs_accept_length}).

\paragraph{Imperfect draft correction.}
With an imperfect draft under uniform per-token mismatch, a ranking error $\arg\max q \neq \arg\max p$ occurs when the gap between the top two target probabilities satisfies $p(v^*_1) - p(v^*_2) \lesssim 2\sigma$.
High-entropy distributions have smaller gaps among top tokens, making ranking errors more frequent.
When a ranking error occurs, the acceptance rate drops from $p(v^*_1)$ to $p(\hat{v}) < p(v^*_1)$, introducing an additional entropy-dependent deficit.
Both effects reinforce the negative slope, so the linear approximation still holds with a potentially steeper slope:
\begin{align}
\alpha^{\mathrm{TO}} \approx a^{\mathrm{TO}} - b^{\mathrm{TO}} \cdot \mathcal{H}(p),
\end{align}
where the slope $b^{\mathrm{TO}}$ is empirically comparable to $b^{\mathrm{RS}}$ (see \S\ref{sec:exp}), though the two arise from different mechanisms: $b^{\mathrm{TO}}$ is driven by the concentration of $\max_y p(y)$ and ranking instability, while $b^{\mathrm{RS}}$ is driven by the accumulation of per-token TV residuals.

\subsection{Rejection Sampling with CE/KL Training}
\label{subsec:ce_rs}

The rejection sampling acceptance rate is $\alpha^{\mathrm{RS}} = 1 - d_{\mathrm{TV}}(p, q)$ (Eq.~\eqref{equ:reject_sample}).
We analyze how CE/KL training produces entropy-dependent acceptance rates through its uniform per-token mismatch structure.

\paragraph{Gradient structure.}
The gradient of $D_{\mathrm{KL}}(p \| q)$ with respect to logits is $\frac{\partial D_{\mathrm{KL}}}{\partial z_j} = q_j - p_j$ (see \S\ref{app:kl_comparison}).
The gradient magnitude $|q_j - p_j|$ is determined by the absolute difference between $p_j$ and $q_j$, not by the magnitude of $p_j$ itself.
Under gradient-based optimization, each token receives optimization pressure proportional to $|q_j - p_j|$, regardless of whether $p_j = 10^{-1}$ or $p_j = 10^{-5}$.
This uniform pressure produces approximately uniform per-token mismatch: $q^*_{\mathrm{CE}}(v) = p(v) + \eta_v$ with $|\eta_v| \lesssim \sigma$.

\paragraph{TV distance derivation.}
Under uniform per-token mismatch:
\begin{align}
d_{\mathrm{TV}}(p, q^*_{\mathrm{CE}}) &= \frac{1}{2}\sum_{v \in \mathcal{V}} |p(v) - q^*(v)| = \frac{1}{2}\sum_{v} |\eta_v|.
\end{align}

The sum decomposes over the effective support $\mathcal{S}_\tau(p)$ and its complement:
\begin{align}
d_{\mathrm{TV}} &= \frac{1}{2}\sum_{v \in \mathcal{S}_\tau} |\eta_v| + \frac{1}{2}\sum_{v \notin \mathcal{S}_\tau} |\eta_v|.
\end{align}

For the complement term, since $p(v) \approx 0$ outside the effective support and $q^*(v) \geq 0$, we have $|\eta_v| = |q^*(v) - p(v)| \leq q^*(v)$, so:
\begin{align}
\frac{1}{2}\sum_{v \notin \mathcal{S}_\tau} |\eta_v| \leq \frac{1}{2}\sum_{v \notin \mathcal{S}_\tau} q^*(v) \leq \frac{1}{2}\left(1 - \sum_{v \in \mathcal{S}_\tau} q^*(v)\right),
\end{align}
which is a small constant independent of $\mathcal{H}(p)$ (since most probability mass concentrates in the effective support for both $p$ and $q^*$).
The entropy-dependent contribution therefore comes from the effective support, where mismatch is fully realized at the $\sigma$ level.
With $|\eta_v| \lesssim \sigma$ for $v \in \mathcal{S}_\tau$ and $|\mathcal{S}_\tau(p)| \approx \exp(\mathcal{H}(p))$:
\begin{align}
d_{\mathrm{TV}}(p, q^*_{\mathrm{CE}}) \approx \frac{\sigma}{2} \cdot \exp(\mathcal{H}(p)).
\end{align}

Therefore:
\begin{align}
\alpha^{\mathrm{RS}}_{\mathrm{CE}} = 1 - d_{\mathrm{TV}} \approx 1 - \frac{\sigma}{2} \exp(\mathcal{H}(p)).
\end{align}

\paragraph{Linear approximation.}
In the regime where $\mathcal{H}(p)$ varies over a moderate range $[\mathcal{H}_{\min}, \mathcal{H}_{\max}]$ (e.g., $[0.1, 0.5]$ during RL training), the exponential can be linearized via a first-order Taylor expansion around $\bar{\mathcal{H}} = \frac{1}{2}(\mathcal{H}_{\min} + \mathcal{H}_{\max})$:
\begin{align}
\exp(\mathcal{H}(p)) \approx \exp(\bar{\mathcal{H}}) \cdot \bigl(1 + (\mathcal{H}(p) - \bar{\mathcal{H}})\bigr).
\end{align}
Substituting:
\begin{align}
\alpha^{\mathrm{RS}}_{\mathrm{CE}} \approx a^{\mathrm{RS}} - b^{\mathrm{RS}} \cdot \mathcal{H}(p),
\end{align}
where $a^{\mathrm{RS}} = 1 - \frac{\sigma}{2}\exp(\bar{\mathcal{H}})(1 - \bar{\mathcal{H}})$ and $b^{\mathrm{RS}} = \frac{\sigma}{2}\exp(\bar{\mathcal{H}})$ are positive constants.
This explains the empirically observed linear negative correlation between entropy and acceptance rate under CE/KL training.

\paragraph{Intuition.}
CE/KL training distributes optimization resources uniformly across all tokens.
When $\mathcal{H}(p)$ is low, $p$ concentrates on a few tokens, and the draft only needs to match these accurately --- the additive errors on the remaining tokens contribute negligibly to $d_{\mathrm{TV}}$.
When $\mathcal{H}(p)$ is high, $p$ spreads across $\exp(\mathcal{H}(p))$ tokens, and the uniform additive errors accumulate into a large TV distance.

\paragraph{Why CE/KL training is suboptimal for rejection sampling.}
Pinsker's inequality states $d_{\mathrm{TV}}(p, q) \leq \sqrt{\frac{1}{2}D_{\mathrm{KL}}(p \| q)}$, relating the two divergences.
However, the suboptimality of CE/KL training for rejection sampling does not stem from the looseness of this bound per se, but from how the KL gradient allocates model capacity across the vocabulary.

Under uniform per-token mismatch, a second-order expansion of $D_{\mathrm{KL}}$ gives:
\begin{align}
D_{\mathrm{KL}}(p \| q^*_{\mathrm{CE}}) &\approx \frac{1}{2}\sum_v \frac{\eta_v^2}{p(v)}, \\
d_{\mathrm{TV}}(p, q^*_{\mathrm{CE}}) &= \frac{1}{2}\sum_v |\eta_v|.
\end{align}
By the Cauchy--Schwarz inequality, $\left(\sum_v |\eta_v|\right)^2 \leq \left(\sum_v \eta_v^2 / p(v)\right)\left(\sum_v p(v)\right)$, which recovers Pinsker's bound $(2\,d_{\mathrm{TV}})^2 \leq 2\,D_{\mathrm{KL}}$.
Equality holds when $|\eta_v| \propto \sqrt{p(v)}$---i.e., the bound is \textit{tightest} when $p$ is uniform.

The fundamental issue is instead one of \textbf{capacity allocation}.
The KL gradient $\partial D_{\mathrm{KL}} / \partial z_j = q_j - p_j$ applies optimization pressure proportional to the absolute difference $|q_j - p_j|$, distributing finite model capacity roughly uniformly across all tokens, including those with negligible target probability.
Under this uniform allocation, each token contributes a uniform mismatch $|\eta_v| \lesssim \sigma$, and the resulting TV distance scales with the number of tokens in the effective support:
\begin{align}
d_{\mathrm{TV}} \approx \frac{\sigma}{2} \cdot |\mathcal{S}_{\mathrm{eff}}(p)| \propto \exp(\mathcal{H}(p)) \cdot \sigma.
\end{align}
High-entropy distributions spread mass across more tokens ($|\mathcal{S}_{\mathrm{eff}}| \approx \exp(\mathcal{H})$), accumulating more per-token residuals into a larger TV distance, even though the KL divergence is also being minimized.
This is why CE/KL-trained drafts exhibit a strong negative entropy--acceptance correlation: the KL objective does not distinguish between tokens that matter for the rejection sampling acceptance decision and those that do not.

\subsection{Rejection Sampling with TV Training}
\label{subsec:tv_rs}

\paragraph{Gradient structure.}
The gradient of the TV loss with respect to logits is $\frac{\partial \mathcal{L}_{\mathrm{TV}}}{\partial z_j} = -q_j [\mathbbm{1}[q_j \leq p_j] - S]$ (Eq.~\eqref{equ:tv_grad}).

\textbf{Key observation:} The gradient is \textit{proportional to $q_j$}. This means:
\begin{itemize}[leftmargin=8mm, itemsep=1mm]
    \item High-probability tokens ($q_j$ large) receive a strong gradient signal and are optimized accurately.
    \item Low-probability tokens ($q_j \approx 0$) receive near-zero gradient, so the optimizer does not waste capacity on them.
\end{itemize}

\paragraph{TV gradient as a self-correcting mechanism.}
Define the probability ratio $r_j = q_j / p_j$. The TV gradient (Eq.~\eqref{equ:tv_grad}) acts as a self-correcting feedback that drives $r_j \to 1$:
\begin{itemize}[leftmargin=8mm, itemsep=1mm]
    \item When $r_j < 1$ (i.e., $q_j < p_j$): the indicator $\mathbbm{1}[q_j \leq p_j] = 1$, so
    \begin{align}
    \frac{\partial \mathcal{L}_{\mathrm{TV}}}{\partial z_j} = -q_j(1 - S) < 0,
    \end{align}
    and gradient descent increases $z_j$, pushing $q_j$ upward and $r_j$ toward $1$.
    \item When $r_j > 1$ (i.e., $q_j > p_j$): the indicator $\mathbbm{1}[q_j \leq p_j] = 0$, so
    \begin{align}
    \frac{\partial \mathcal{L}_{\mathrm{TV}}}{\partial z_j} = q_j \cdot S > 0,
    \end{align}
    and gradient descent decreases $z_j$, pushing $q_j$ downward and $r_j$ toward $1$.
\end{itemize}
In both cases, TV training drives $r_j \to 1$, i.e., $\log(q_j/p_j) \to 0$.
Moreover, since $q_j \ll 1$ for typical vocabulary sizes, the softmax locally satisfies $\partial q_j / \partial z_j \approx q_j$, so a single gradient step produces
\begin{align}
\label{equ:tv_self_correct}
\Delta(\log r_j) \approx \Delta z_j =
\begin{cases}
\eta \, q_j (1 - S) > 0 & \text{if } r_j < 1, \\
-\eta \, q_j \, S < 0 & \text{if } r_j > 1,
\end{cases}
\end{align}
where $\eta$ is the learning rate.
The correction magnitude is proportional to $q_j$: tokens with larger probability receive a stronger corrective signal, ensuring that $|\log r_j|$ on the effective support converges to a bounded value $\epsilon$.
Tail tokens ($q_j \approx 0$) receive negligible correction but also contribute negligible TV distance.

This self-correcting dynamics contrasts with CE/KL training, whose gradient $\partial D_{\mathrm{KL}} / \partial z_j = q_j - p_j$ drives \textit{absolute} differences $|q_j - p_j|$ toward zero uniformly, rather than \textit{ratios} $q_j/p_j$ toward one.
Under finite capacity, the CE/KL equilibrium maintains $|q_j - p_j| \lesssim \sigma$ uniformly, which corresponds to $|q_j/p_j - 1| \lesssim \sigma / p_j$---an unbounded ratio for small-$p_j$ tokens in the effective support.

\paragraph{Assumption: bounded logit-ratio error on the effective support.}
The self-correcting property above motivates the following assumption.
Let $q^*_{\mathrm{TV}}$ be the solution reached by TV training under finite draft capacity.
Since the correction magnitude in Eq.~\eqref{equ:tv_self_correct} is proportional to $q_j \approx p_j$ for tokens in the effective support, these tokens receive sufficient gradient signal to drive $\log r_j$ into a bounded interval.

The assumption is stated in \textit{log-ratio} space ($|\log(q/p)| \leq \epsilon$) rather than absolute space ($|q - p| \leq \sigma$) because gradient descent operates on logits $z_j$, and the softmax satisfies $\log q_j = z_j - \log Z$, so each logit update $\Delta z_j$ directly translates to $\Delta(\log q_j) \approx \Delta z_j$.
Since $p$ is fixed, $\Delta(\log r_j) = \Delta(\log q_j) \approx \Delta z_j$: the optimizer's native space is log-ratio, and the equilibrium error is therefore naturally bounded in log-ratio.

We assume: there exists a constant $\epsilon$ such that, for all $j \in \mathcal{S}_{\mathrm{eff}}(p)$,
\begin{align}
\left|\log \frac{q^*_{\mathrm{TV}}(j)}{p_j}\right| \leq \epsilon.
\end{align}
Tail tokens may have larger relative uncertainty but carry negligible probability mass and contribute negligible TV distance.

\paragraph{Deriving the mismatch bound.}
The bounded logit-ratio assumption implies
\begin{align}
e^{-\epsilon}p_j \leq q^*_{\mathrm{TV}}(j) \leq e^{\epsilon}p_j.
\end{align}
Therefore, for every token in the effective support,
\begin{align}
|q^*_{\mathrm{TV}}(j) - p_j|
&= p_j\left|\frac{q^*_{\mathrm{TV}}(j)}{p_j} - 1\right| \\
&\leq p_j\max\{e^{\epsilon}-1,\,1-e^{-\epsilon}\} \\
&= (e^{\epsilon}-1)p_j.
\end{align}
Letting $\delta = e^{\epsilon}-1$, we obtain
\begin{align}
\boxed{|q^*_{\mathrm{TV}}(j) - p_j| \lesssim \delta\,p_j.}
\end{align}
That is, under the bounded-logit-ratio assumption induced by the TV gradient's capacity allocation, TV training yields probability-proportional mismatch ($|q-p| \lesssim \delta \cdot p$) rather than the uniform mismatch ($|q-p| \lesssim \sigma$) of CE/KL training.
In practice, optimizer dynamics (e.g., Adam's second-moment normalization) may partially attenuate the raw $q_j$-proportionality, so the proportional mismatch should be viewed as a modeling approximation rather than an unconditional theorem.

\paragraph{TV distance derivation.}
Under probability-proportional mismatch with constant $\delta$:
\begin{align}
d_{\mathrm{TV}}(p, q^*_{\mathrm{TV}}) &= \frac{1}{2}\sum_{v} |p(v) - q^*(v)| \leq \frac{\delta}{2}\sum_{v} p(v) = \frac{\delta}{2}.
\end{align}

This bound is independent of $\mathcal{H}(p)$, yielding:
\begin{align}
\alpha^{\mathrm{RS}}_{\mathrm{TV}} \geq 1 - \frac{\delta}{2},
\end{align}
which proves Proposition~\ref{prop:tvd_entropy}.

\paragraph{Practical considerations.}
The above analysis assumes $\delta$ is a constant, but in practice, the draft head has finite capacity.
When $\mathcal{H}(p)$ increases, the effective support $|\mathcal{S}_{\mathrm{eff}}(p)| \approx \exp(\mathcal{H}(p))$ grows, and maintaining uniform relative accuracy across more tokens may require more model capacity.
If the draft head's capacity is insufficient, $\delta$ may exhibit weak entropy dependence $\delta = \delta(\mathcal{H})$, reintroducing a residual (but substantially attenuated) entropy--acceptance correlation.
Empirically, the entropy--acceptance slope under TV training is reduced by over $95\%$ compared to CE/KL training (e.g., $-0.06$ vs.\ $-1.68$), confirming that the probability-proportional mismatch largely holds but is not perfect.

\paragraph{Intuition.}
TV training allocates optimization resources proportionally to each token's probability.
When $\mathcal{H}(p)$ is high and the distribution spreads across many tokens, each token receives proportionally less optimization effort, but also carries proportionally less weight in the TV distance.
These two effects largely cancel, making the entropy--acceptance relationship substantially weaker than under CE/KL training.

\section{Rejection Sampling Decision Boundary Derivation}
\label{app:rs_boundary}

We derive the condition under which rejection sampling achieves a higher acceptance rate than target-only sampling.

\paragraph{Acceptance rates.}
Under target-only sampling, the acceptance rate is $\alpha^{\mathrm{TO}} = p(\hat{y})$, where $\hat{y} = \arg\max_y q(y)$ is the draft's top-1 token.
Under rejection sampling, the acceptance rate is:
\begin{align}
\alpha^{\mathrm{RS}} = \sum_{v \in \mathcal{V}} \min\bigl(p(v), q(v)\bigr).
\end{align}

\paragraph{Decomposing $\alpha^{\mathrm{RS}}$.}
Using the identity $\min(a, b) = \frac{1}{2}(a + b - |a - b|)$ and the normalization $\sum_v p(v) = \sum_v q(v) = 1$:
\begin{align}
\alpha^{\mathrm{RS}} &= \sum_v \frac{p(v) + q(v) - |p(v) - q(v)|}{2} \\
&= \frac{1}{2}\!\sum_v p(v) + \frac{1}{2}\!\sum_v q(v) - \frac{1}{2}\!\sum_v |p(v) - q(v)| \\
&= 1 - d_{\mathrm{TV}}(p, q).
\end{align}

\paragraph{Decision boundary.}
RS outperforms target-only when $\alpha^{\mathrm{RS}} > \alpha^{\mathrm{TO}}$:
\begin{align}
1 - d_{\mathrm{TV}}(p, q) > p(\hat{y}) \quad \Longleftrightarrow \quad d_{\mathrm{TV}}(p, q) < 1 - p(\hat{y}).
\end{align}
This reduces the comparison between the two acceptance methods to a simple inequality: RS is preferred whenever the draft--target TVD is smaller than the target probability mass outside the draft's greedy prediction.
Since $1 - p(\hat{y}) \geq 1 - \max_y p(y) > 0$ for any non-degenerate distribution, there always exists a sufficiently well-aligned draft for which RS is beneficial.

\section{Fused TV Loss Kernel}
\label{app:tv_kernel}

We provide the pseudocode for our fused TV loss implementation.
The forward pass (Algorithm~\ref{alg:tv_fwd}) computes the per-token TV loss and the auxiliary quantity $S$ needed by the backward pass in a single kernel launch.
The backward pass (Algorithm~\ref{alg:tv_bwd}) computes gradients with respect to the draft logits.
Both kernels iterate over the vocabulary in tiles of size \texttt{BLOCK\_V} to bound register and shared-memory usage, enabling full-vocabulary TV loss computation without materializing the softmax output.

\begin{algorithm}[h]
\caption{TV Loss Forward Kernel (per token position)}
\label{alg:tv_fwd}
\begin{algorithmic}[1]
\Require Draft logits $z \in \mathbb{R}^{|\mathcal{V}|}$, target log-probs $\log p \in \mathbb{R}^{|\mathcal{V}|}$
\Ensure TV loss $\ell$, auxiliary scalar $S$
\State \textit{// Pass 1: numerically stable softmax denominator}
\State $m \gets \max_v z_v$ \Comment{global logit max}
\State $D \gets \sum_v \exp(z_v - m)$ \Comment{exp-sum}
\State \textit{// Pass 2: tiled overlap and $S$ accumulation}
\State $\textit{overlap} \gets 0$; \quad $S \gets 0$
\For{$v_{\text{start}} = 0$ \textbf{to} $|\mathcal{V}|$ \textbf{step} \texttt{BLOCK\_V}}
    \State $\mathbf{v} \gets [v_{\text{start}},\ldots,v_{\text{start}}+\texttt{BLOCK\_V}{-}1]$
    \State $\mathbf{q} \gets \exp(\mathbf{z}[\mathbf{v}] - m) \;/\; D$ \Comment{draft prob}
    \State $\mathbf{p} \gets \exp(\log\mathbf{p}[\mathbf{v}])$ \Comment{target prob}
    \State $\textit{overlap} \mathrel{+}= \sum \min(\mathbf{q},\; \mathbf{p})$
    \State $S \mathrel{+}= \sum \mathbf{q} \cdot \mathbbm{1}[\mathbf{q} \leq \mathbf{p}]$
\EndFor
\State $\ell \gets \text{clamp}(1 - \textit{overlap},\; 0,\; \tau_{\max})$ \Comment{$\tau_{\max}$: optional clamp}
\State \Return $\ell,\; S$
\end{algorithmic}
\end{algorithm}

\begin{algorithm}[h]
\caption{TV Loss Backward Kernel (per token position)}
\label{alg:tv_bwd}
\begin{algorithmic}[1]
\Require Draft logits $z$, target log-probs $\log p$, cached $(m, D, S, g_{\text{out}})$
\Ensure Gradient $\nabla_z \ell \in \mathbb{R}^{|\mathcal{V}|}$
\For{$v_{\text{start}} = 0$ \textbf{to} $|\mathcal{V}|$ \textbf{step} \texttt{BLOCK\_V}}
    \State $\mathbf{v} \gets [v_{\text{start}},\ldots,v_{\text{start}}+\texttt{BLOCK\_V}{-}1]$
    \State $\mathbf{q} \gets \exp(\mathbf{z}[\mathbf{v}] - m) \;/\; D$
    \State $\mathbf{p} \gets \exp(\log\mathbf{p}[\mathbf{v}])$
    \State $\nabla_z[\mathbf{v}] \gets \mathbf{q} \cdot (S - 1 + \mathbbm{1}[\mathbf{q} > \mathbf{p}]) \cdot g_{\text{out}}$
\EndFor
\State \Return $\nabla_z$
\end{algorithmic}
\end{algorithm}

\paragraph{Implementation notes.}
(1) The forward kernel fuses the softmax normalization with the TV overlap computation, avoiding a separate $O(|\mathcal{V}|)$ softmax pass.
(2) For tensor-parallel training, $m$ and $D$ are computed via \texttt{all\_reduce} across TP ranks before the overlap pass; the local overlaps and $S$ values are similarly reduced after computation.
(3) The optional \texttt{top-K} path selects the $K$ largest draft logits and computes TV/gradients only at those positions, reducing memory from $O(|\mathcal{V}|)$ to $O(K)$ with negligible accuracy loss (since the gradient $\propto q_j \approx 0$ for tail tokens).

\section{Rejection Sampling Inference Implementation}
\label{app:rs_inference}

Implementing rejection sampling for MTP-based speculative decoding in production inference engines requires modifying both the draft and verification stages.
Unlike target-only sampling, which selects draft tokens via $\arg\max$ and accepts based solely on the target probability, rejection sampling requires (1) sampling draft tokens from the draft distribution $q$ (rather than taking the argmax), (2) caching the draft probabilities for use during verification, and (3) computing the acceptance ratio $\min(1, p(\hat{y})/q(\hat{y}))$ during verification.
We describe two different implementation strategies as follows.

\subsection{Multinomial Draft Sampling (SGLang)}
\label{app:rs_sglang}

The first approach, implemented in SGLang\footnote{\url{https://github.com/sgl-project/sglang/pull/26312}}, directly samples draft tokens from the draft distribution using multinomial sampling.

\paragraph{Draft stage.}
Instead of selecting draft tokens via $\hat{y} = \arg\max_y q(y)$, we apply temperature scaling to the draft logits and sample $\hat{y} \sim q(\cdot)$ via multinomial sampling.
The full draft probability vector $q \in \mathbb{R}^{|\mathcal{V}|}$ is cached alongside each draft token for use during verification.

\paragraph{Verification stage.}
Given a chain of $\gamma$ draft tokens $\hat{y}_1, \ldots, \hat{y}_\gamma$ with cached draft probabilities $q_1, \ldots, q_\gamma$, and the target probabilities $p_1, \ldots, p_\gamma$ obtained from the single-pass target model verification, we implement rejection sampling via a fused Triton kernel.
The kernel processes each request independently (one Triton program per request) and performs two phases:

\begin{enumerate}[leftmargin=8mm, itemsep=1mm]
    \item \textbf{Sequential acceptance}: For each draft step $i = 1, \ldots, \gamma$, draw $u_i \sim \mathrm{Uniform}(0, 1)$ and accept $\hat{y}_i$ if
    $u_i \cdot q_i(\hat{y}_i) < p_i(\hat{y}_i)$,
    i.e., with probability $\min\!\bigl(1,\, p_i(\hat{y}_i) / q_i(\hat{y}_i)\bigr)$.
    Stop at the first rejection.

    \item \textbf{Residual resampling}: If draft token $\hat{y}_j$ is rejected at step $j$, or if all $\gamma$ drafts are accepted (bonus token case), sample the next token from the residual distribution.
    For rejection at step $j$, the residual distribution is $p_{\mathrm{resid}}(v) \propto \max(0, p_j(v) - q_j(v))$; for the bonus token (all accepted), the residual is simply $p_{\gamma}(v)$.
    The kernel computes this via a two-pass CDF inversion over the vocabulary: Pass~1 computes the normalization constant $Z = \sum_v \max(0, p_j(v) - q_j(v))$, and Pass~2 finds the token $v^*$ such that the cumulative sum first exceeds $u \cdot Z$ for a uniform random $u$.
\end{enumerate}

\begin{algorithm}[h]
\caption{Chain Rejection Sampling Verification (Multinomial / SGLang)}
\label{alg:rs_multinomial}
\begin{algorithmic}[1]
\Require Draft tokens $\hat{y}_1, \ldots, \hat{y}_\gamma$; draft probs $q_1, \ldots, q_\gamma \in \mathbb{R}^{|\mathcal{V}|}$; target probs $p_1, \ldots, p_\gamma \in \mathbb{R}^{|\mathcal{V}|}$
\Ensure Accepted token count $n$; output token $y^*$ at position $n+1$
\State $n \gets \gamma$ \Comment{assume all accepted}
\For{$i = 1$ \textbf{to} $\gamma$}
    \State $u_i \sim \mathrm{Uniform}(0, 1)$
    \If{$u_i \cdot q_i(\hat{y}_i) \geq p_i(\hat{y}_i)$} \Comment{reject}
        \State $n \gets i - 1$; \textbf{break}
    \EndIf
\EndFor
\State \textit{// Residual resampling via two-pass CDF inversion}
\If{$n < \gamma$} \Comment{rejected at step $n+1$}
    \State $r(v) \gets \max\bigl(0,\; p_{n+1}(v) - q_{n+1}(v)\bigr)$ for all $v$
\Else \Comment{bonus token}
    \State $r(v) \gets p_\gamma(v)$ for all $v$
\EndIf
\State $Z \gets \sum_v r(v)$ \Comment{Pass 1: normalization}
\State $u \sim \mathrm{Uniform}(0, 1)$
\State $y^* \gets \min\bigl\{v : \sum_{v' \leq v} r(v') \geq u \cdot Z\bigr\}$ \Comment{Pass 2: CDF inversion}
\State \Return $n,\; y^*$
\end{algorithmic}
\end{algorithm}

\paragraph{Memory overhead.}
The primary overhead is caching the draft probability vectors: $O(\gamma \times |\mathcal{V}|)$ per request, where $\gamma$ is the number of MTP steps.

\subsection{Gumbel-Max Trick (vLLM)}
\label{app:rs_vllm}

The second approach, implemented in vLLM\footnote{\url{https://github.com/vllm-project/vllm/pull/35461}}, avoids explicit CDF inversion during residual resampling by leveraging the Gumbel-Max trick.

\paragraph{Draft stage.}
Draft tokens are sampled using the Gumbel-Max trick: for each vocabulary token $v$, compute $v^* = \arg\max_v [\log q(v) / \tau + G_v]$, where $G_v \sim \mathrm{Gumbel}(0, 1)$ is i.i.d.\ Gumbel noise and $\tau$ is the sampling temperature.
This is equivalent to sampling from $q$ after temperature scaling.
The temperature-scaled draft logits (before adding Gumbel noise) are cached for verification.

\paragraph{Verification stage.}
The verification is split into two kernels:

\begin{enumerate}[leftmargin=8mm, itemsep=1mm]
    \item \textbf{Acceptance kernel}: A sequential Triton kernel iterates over draft steps, computing $p(\hat{y}_i)$ and $q(\hat{y}_i)$ from the cached target and draft probabilities, and accepting if $u_i \cdot q(\hat{y}_i) < p(\hat{y}_i)$ for a pseudo-random $u_i$ generated via \texttt{tl.rand} seeded by the request's random seed and position.
    The kernel records the index of the first rejected step.

    \item \textbf{Residual logits kernel}: A parallel Triton kernel computes the residual distribution in logit space.
    For rejection at step $j$: $z_{\mathrm{resid}}(v) = \log \max(0, p_j(v) - q_j(v))$; for the bonus token: $z_{\mathrm{resid}}(v) = z_{\mathrm{target},\gamma}(v)$ (the raw target logits).
    The resampled token is then drawn from this residual distribution using the same Gumbel-Max sampling as the draft stage.
\end{enumerate}

\begin{algorithm}[h]
\caption{Chain Rejection Sampling Verification (Gumbel-Max / vLLM)}
\label{alg:rs_gumbel}
\begin{algorithmic}[1]
\Require Draft tokens $\hat{y}_1, \ldots, \hat{y}_\gamma$; draft logits $z^q_1, \ldots, z^q_\gamma \in \mathbb{R}^{|\mathcal{V}|}$; target probs $p_1, \ldots, p_\gamma \in \mathbb{R}^{|\mathcal{V}|}$; target logits $z^p_\gamma$
\Ensure Accepted token count $n$; output token $y^*$ at position $n+1$
\State \textit{// Kernel 1: sequential acceptance}
\State $n \gets \gamma$
\For{$i = 1$ \textbf{to} $\gamma$}
    \State $q_i(\hat{y}_i) \gets \mathrm{softmax}(z^q_i)_{\hat{y}_i}$
    \State $u_i \gets \texttt{tl.rand}(\textit{seed}, i)$
    \If{$u_i \cdot q_i(\hat{y}_i) \geq p_i(\hat{y}_i)$}
        \State $n \gets i - 1$; \textbf{break}
    \EndIf
\EndFor
\State \textit{// Kernel 2: residual logits}
\If{$n < \gamma$}
    \State $z_{\mathrm{resid}}(v) \gets \log \max\bigl(0,\; p_{n+1}(v) - q_{n+1}(v)\bigr)$ for all $v$
\Else
    \State $z_{\mathrm{resid}}(v) \gets z^p_\gamma(v)$ for all $v$
\EndIf
\State \textit{// Gumbel-Max resampling}
\State $G_v \sim \mathrm{Gumbel}(0, 1)$ for all $v$
\State $y^* \gets \arg\max_v \bigl[z_{\mathrm{resid}}(v) + G_v\bigr]$
\State \Return $n,\; y^*$
\end{algorithmic}
\end{algorithm}

\end{document}